\documentclass{article}

\usepackage[final]{neurips_2025}

\usepackage[utf8]{inputenc} 
\usepackage[T1]{fontenc}    
\usepackage{hyperref}       
\usepackage{url}            
\usepackage{booktabs}       
\usepackage{amsfonts}       
\usepackage{nicefrac}       
\usepackage{microtype}      
\usepackage{xcolor}         

\usepackage{authblk}
\usepackage{graphicx}
\usepackage{amsmath}
\usepackage{amsfonts}
\usepackage{amssymb}
\usepackage{amsthm}
\usepackage{mathtools}
\usepackage{enumerate}
\usepackage{lscape}
\usepackage{longtable}
\usepackage{rotating}
\usepackage{multirow}
\usepackage{color}
\usepackage{url}
\usepackage{subfigure}
\usepackage{rotating}
\usepackage{comment}
\usepackage[utf8]{inputenc}
\usepackage{bbm}
\usepackage{hyperref}

\usepackage{float}
\usepackage{caption}
\usepackage{wrapfig}
\usepackage{algorithm}
\usepackage{algorithmic}

\usepackage{tocbibind} 
\usepackage{titletoc}  

\newcommand{\RR}{ \mathbb{R} }

\newcommand{\spaceo}{\hspace{2 mm}}
\newcommand{\setsep}{ \spaceo | \spaceo}
\newcommand{\half}{\frac{1}{2}}

\newcommand{\Prob}[1]{\mathbb{P}\left( #1 \right)}

\newcommand{\Probu}[2]{\mathbb{P}_{#1}\left( #2 \right)}

\newcommand{\Abs}[1]{\left| #1 \right|}
\newcommand{\Set}[1]{\left\{ #1 \right\}}
\newcommand{\Brack}[1]{\left( #1 \right)}

\newcommand{\SqBrack}[1]{\left[ #1 \right]}

\newcommand{\Exp}[1]{ \mathbb{E} #1}
\newcommand{\Expsubidx}[2]{ \mathbb{E}_{#1} #2}

\newcommand{\norm}[1]{\left\|#1\right\|}

\newcommand{\Ind}[1]{ \mathbbm{1}_{\Set{#1}} }
\newcommand{\eps}{\varepsilon}

\newlength{\dhatheight}

\newcommand{\cond}{\vert}
\newcommand\indep{\protect\mathpalette{\protect\independenT}{\perp}}
\def\independenT#1#2{\mathrel{\rlap{$#1#2$}\mkern2mu{#1#2}}}

\newcommand{\mcX}{\mathcal{X}}

\newcommand{\mcY}{\mathcal{Y}}

\newcommand{\mcZ}{\mathcal{Z}}
\newcommand{\mcA}{\mathcal{A}}

\newcommand{\notemk}[1]{{\color{blue}{ #1 }}}

\newcommand{\simplxY}{\Delta_{\mcY}}

\newtheorem{theorem}{Theorem}[section]
\newtheorem{definition}{Definition}[section]
\newtheorem{lemma}{Lemma}[section]

\author[1]{Mark Kozdoba \textsuperscript{*}}
\author[1]{Binyamin Perets \textsuperscript{*}}
\author[1,2]{Shie Mannor}

\affil[1]{Technion - Israel Institute of Technology, Haifa, Israel}
\affil[2]{NVIDIA}

\title{Efficient Fairness-Performance Pareto Front Computation}

\begin{document}
\maketitle

\begin{abstract}
There is a well known intrinsic trade-off between the fairness of a representation and the performance of classifiers derived from the representation. In this paper we propose a new method to compute the optimal Pareto front of this trade off. In contrast to the existing methods, this approach does not require the training of complex fair representation models.

Our approach is derived through three main steps: We analyze fair representations theoretically, and derive several structural properties of optimal representations. We then show that these properties enable a reduction of the computation of the Pareto Front to a compact discrete problem. Finally, we show that these compact approximating problems can be efficiently solved via off-the shelf concave-convex programming methods.

In addition to representations, we show that the new methods may also be used to directly compute the Pareto front of fair classification problems. Moreover, the proposed methods may be used with any concave performance measure. This is in contrast to the existing reduction approaches, developed recently in fair classification, which rely explicitly on the structure of the non-differentiable accuracy measure, and are thus unlikely to be extendable. 

The approach was evaluated on several real world benchmark datasets and compares favorably to a number of recent state of the art fair representation and classification methods.  
\end{abstract}

\section{Introduction}
\label{sec:introduction}
Fair representations are a central topic in the field of Fair Machine Learning, \cite{mehrabi2021survey}, \cite{pessach2022review},\cite{chouldechova2018frontiers}. Since their introduction in \cite{zemel2013learning_fair}, Fair representations have been extensively studied, giving rise to a variety of approaches based on a wide range of modern machine learning methods, such GANs, variational auto encoders, numerous variants of Optimal Transport methods, and direct variational formulations.  See the papers 
\cite{feldman2015certifying}, \cite{madras2018learning}, \cite{gordaliza2019obtaining, zehlike2020matching}, \cite{song2019learning}, \cite{du2020fairness}, \cite{zhao2022inherent}, \cite{jovanovic2023fare}, \cite{dehdashtian2024utility}, for a sample of existing methods.

For a given representation learning problem and a target classification problem, since the fairness constraints reduce the space of feasible classifiers, the best possible classification performance will usually be lower as the fairness constraint becomes stronger. This phenomenon is known as the Fairness-Performance trade-off. 
Assume that we have fixed a way to measure fairness. Then for a given representation learning method, one is often interested in the fairness-performance curve $(\gamma, E(\gamma))$. Here, $\gamma$ is the fairness level, and $E(\gamma)$ is the classification performance of the method at that level. 
The curve $(\gamma, E(\gamma))$ where $E(\gamma)$ is the best possible performance over all representations and classifiers under the constraint is known as the Fairness-Performance Pareto Front.

As indicated by the above discussion, representation learning methods typically involve models with high dimensional parameter spaces, and complex, possibly constrained non-convex optimisation algorithms. 
As such, these methods may be prone to local minima and sensitivity to a variety of hyper-parameters, such as architecture details, learning rates, and even initializations. While the representations produced by such methods may often be useful, it nevertheless may be difficult to decide whether  
their associated Fairness-Performance curve is close to the true Pareto Front.

In this paper we propose a new method to compute the optimal Pareto front, 
which does not require the training of complex fair representation models. In other words, we show that, perhaps somewhat surprisingly, the computation of the Pareto Front can be decoupled from that of the representation, and only relies on learning of much simpler, unconstrained classifiers on the data. To achieve this, we first show that the optimal fair representations satisfy a number of structural properties. While these properties may be of independent interest, here we use them to express the points on the Pareto Front as the solutions of small discrete optimisation problems. These problems, known as \emph{concave minimisation} problems \cite{benson1995concave},  have been extensively studied and can be efficiently solved using modern dedicated optimisation frameworks, \cite{shen2016disciplined}.

We now describe the results in more detail. Let $X \in \mcX$ and $A\in \mcA$ denote the data features and the sensitive attribute, respectively. We assume that $A$ is binary, while $X$ may take values in an arbitrary space $\mcX$, typically with $\mcX = \RR^d$. In addition, we have a \emph{target} variable $Y$, taking values in a finite set $\mcY$. We are then interested in representations that maximise the performance of prediction of $Y$, under the fairness constraint. The representation is denoted by $Z$, and is typically expressed by constructing the conditional distributions $\Prob{Z \cond X=x,A=a}$, for $a\in \Set{0,1}$. The problem setting is illustrated in Figure \ref{fig:representations_main}. 

\begin{figure*}[t] 
    \centering
    \includegraphics[width=0.8\textwidth, height = 3.3cm]{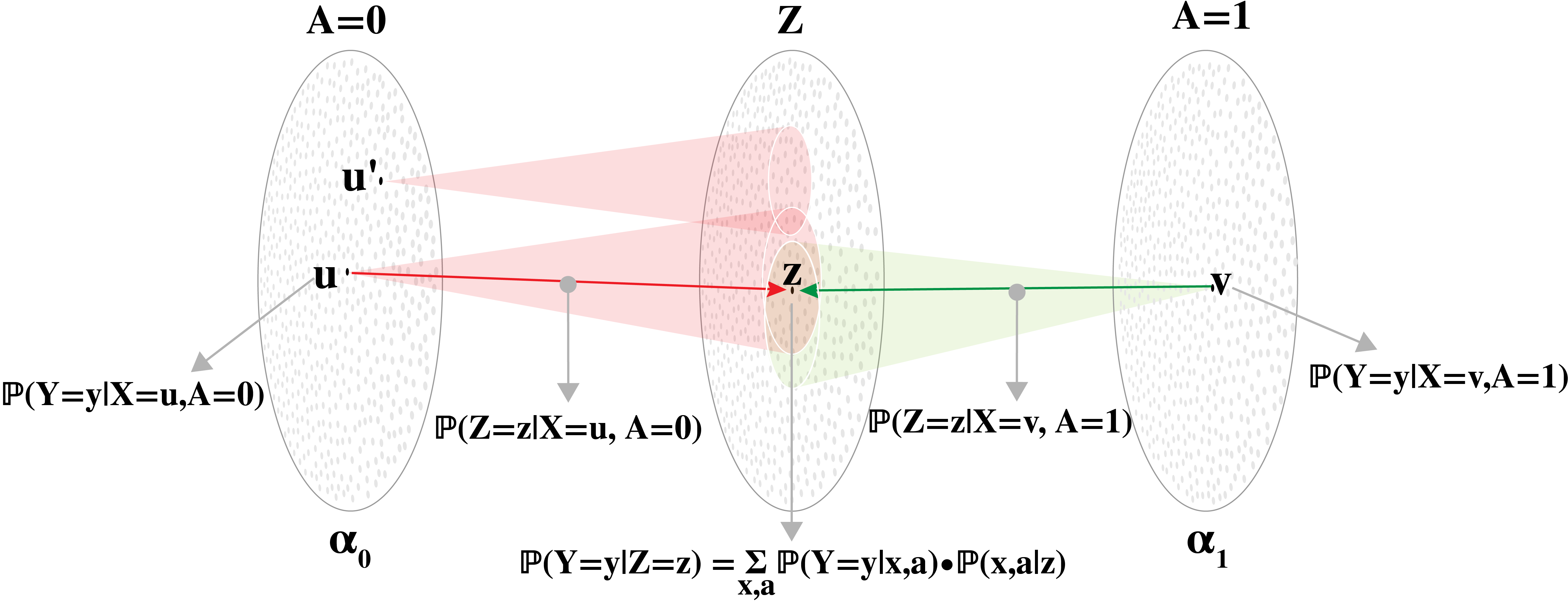} 
    \caption{Fair Representation Problem Setting.}
    \label{fig:representations_main} %
\end{figure*}

Our approach consists of three main steps: We first observe that one can map the data features $(x,a)\in \mcX \times \mcA$ to a much smaller space $\Delta_{\mcY}$ of distributions on the set of label values $\mcY$, without loosing any information necessary for the computation of the Pareto front. The mapping is done by means of the optimal Bayes classifier. 
This result is referred to as Factorization Lemma, Section \ref{sec:classification_and_factorisation}, where the mapping is done via the optimal Bayes classifier.  
Similar arguments 
were recently implicitly used in the study of fair classification tradeoffs, \citep{xian2023fair,wang2023aleatoric},
but were restricted to classification and to accuracy loss (Section \ref{sec:literature}).

The advantage of working with a small space such as 
$\Delta_{\mcY}$ is that it can be easily discretized. For instance, if $Y={0,1}$, then we essentially have 
$\Delta_{\mcY} = [0,1]$, which is descretised trivially. 
We note that other, data dependent discretisation schemes, such as clustering, maybe possible for problems involving highly multi label targets. Alternatively, one can also consider the dataset itself as a grid, a view that is typically taken by
transportation based approaches, e.x. \cite{gordaliza2019obtaining}, \cite{xian2023fair}.

Next, assuming the data is discretized and finite, we ask 
how large the represntation space $\mcZ$ should be, in order to 
support both optimal performance and fairness? For instance, 
we believe the answer to the following question is not apriori obvious: Can representations on infinite spaces be approximated, in terms of performance and fairness, by representations on finite and bounded spaces $\mcZ$? These questions are addressed by the Invertibility Theorem, Section \ref{sec:opt_reprs_properties}, which asserts that all optimal representations may be taken in a certain canonical form, which we term \emph{invertible}. This result, in conjunction with an additional approximation lemma, is used in Section \ref{sec:mifpo_construction} to construct representations with any desired degree of approximation.

Finally, based on these results in Section \ref{sec:mifpo_construction} we also introduce the MIFPO (Model Independent Fairness-Performance Optimization), a discrete optimisation problem that is essentially equivalent to a computation of the fairness-performance tradeoff on a discrete set. We show that in this situation MIFPO is a concave minimisation problem with linear constraints, and we solve it using the disciplined convex-concave programming framework, DCCP, \cite{shen2016disciplined}. 



We evaluate our approach on standard fairness benchmark datasets and compare its fairness-performance curve to multiple state-of-the-art fair representation methods. We also compare MIFPO to the fairness-performance Pareto front of fair \emph{classifiers}\footnote{See Sections \ref{sec:literature} and \ref{sec:classification_and_factorisation} for the relation between our representation framework and fair classification results.}. As expected, MIFPO effectively serves as an upper bound on almost all other algorithms in both cases. 

To summarise, the contributions of this paper are as follows: \textbf{(a)} We derive several new structural properties of optimal fair representations. 
\textbf{(b)} We use these properties to construct a model independent problem, MIFPO, which can approximate the Pareto Front of arbitrary high dimensional data distributions, but is much simpler to solve than direct representation learning for such distributions. \textbf{(c)} We illustrate the approach on real world fairness benchmarks. 

The rest of this paper is organised as follows: Section \ref{sec:literature} discusses the literature and related work. 
In Section \ref{sec:full_theory_section} we discuss the theoretical results, including factorization and the Invertibility Theorem. 
The MIFPO problem construction and the full Pareto Front computation algorithm are provided in Section \ref{sec:reduction_general}. Experimental results are presented in Section \ref{sec:experiments}, and we conclude the paper in Section \ref{sec:conclusion}. All proofs are provided in the Supplementary Material.

\section{Literature and Prior Work}
\label{sec:literature}
We refer to the book \cite{barocas2023fairness}, and surveys \cite{mehrabi2021survey},\cite{du2020fairness}, for a general overview of representations. Tradeoffs in particular where explicitly studied in \cite{song2019learning}, \cite{balunovic2022fair}, \cite{zhao2022inherent}, \cite{jovanovic2023fare},\cite{dehdashtian2024utility}, among others. 

In this paper we use the total variation based fairness constraints, similarly to the line of work in \cite{madras2018learning}, \cite{zhao2022inherent}, \cite{balunovic2022fair}, \cite{jovanovic2023fare}. 
Other constraints used in the literature include entropy based constraints,  \cite{song2019learning}, or RKHS based independence test constraints, \cite{dehdashtian2024utility}.

As discussed earlier, the vast majority of the work above concentrates on finding neural network based fair representations via involved optimization schemes with possible local minima, which may be hard to analyze. This highlights the usefulness of our approach direct approach to the computation of the Pareto front, which has clear theoretical grounding, and in which sources of approximation error are well understood and may be controlled. 

Relations between fairness and performance were studied in 
\cite{zhao2022inherent}. In particular, for \emph{perfectly fair} representations, they derived lower bounds on the accuracy in terms of the difference of the base rates between the groups. However, this work did not introduce new algorithms for the computation of fair representations or of the associated Pareto front. The extension of the considerations in this paper to the full front was carried in \cite{xian2023fair} for classification (see below).

The accuracy fairness tradeoff has also been extensively studied in the context of fair classification (without  representations), see for instance \cite{agarwal2018reductions},
\cite{kim2020fact} , \cite{alghamdi2022beyond}, \cite{xian2023fair}, \cite{wang2023aleatoric}, for a sample of recent approaches.  In particular, the papers \cite{xian2023fair}, \cite{wang2023aleatoric} are state of the art, and are also the most closely related to our methods, among the existing work.

Similarly to our approach, the analysis in these two papers starts with the estimation of the probabilities $\Prob{Y\cond X,A}$, which are then used to compute the constrained performance. 
However, the subsequent steps are different. Crucially, the analysis in both \cite{xian2023fair} and \cite{wang2023aleatoric} relies critically on the properties of the accuracy as the performance metric. Consequently,  it can not be extended to general concave performance measures, such as the standard (minus) log loss, for instance.  Roughly speaking, in the appropriate sense, accuracy largely ignores classification probabilities. This allows the simple description of classifiers as small confusion matrices in \cite{wang2023aleatoric} (extending the approach of \cite{kim2020fact}), and the restriction of the distributions to the vertices of the simplex in \cite{xian2023fair}. The special structure of accuracy is highlighted also in our
Lemma \ref{lemma:cls}, where we show that classifiers with accuracy may be effectively described by representations using only 2 points. We conclude that even when restricted to classification, our approach analyses a fundamentally more complex situation compared to previous work. 
On the other hand, \cite{xian2023fair} and \cite{wang2023aleatoric} support non binary sensitive attributes and group labels, while such an extension for our methods is out of scope for this paper. 
A comparison of computational complexities for these algorithms may be found in Supplementary \ref{sec:computational_complexities}.

\section{Structure of Fair Representations}
\label{sec:full_theory_section}
In this Section we describe several theoretical properties of fair representations. In Section \ref{sec:problem_setting} we introduce the problem setup and the necessary notation. In Section \ref{sec:classification_and_factorisation} we discuss relations to classification with accuracy loss and the factorization result, which allows to reduce the size of the representation space. The Invertibility Theorem is introduced in Section \ref{sec:opt_reprs_properties}.

\subsection{Problem Setting}
\label{sec:problem_setting}
Let $A$ be a binary sensitive variable, and let $X$ be an additional feature random variable, with values in a set 
$\mcX$, typically with $\mcX = \RR^d$. Assume also that there is a target variable $Y$ with finitely many values in a set $\mcY$, jointly distributed with $X,A$. 

A representation $Z$ of $(X,A)$ is defined as a random variable taking values in some space $\mcZ$, with \emph{(i)} distribution given through 
$\Probu{\theta}{Z \cond X,A}$, where $\theta$ are \emph{the parameters of the representation}, and \emph{(ii)} such that conditioned on $(X,A)$, $Z$ is independent of the rest of the variables of the problem. In particular, we have 
\begin{equation}
\label{eq:representation_y_indep_condition}
    Z \indep Y  \spaceo \cond (X,A), 
\end{equation}
where $\indep$ denotes statistical independence. 

Fairness in this paper will be measured by the Total Variation distance. For two distributions, $\mu,\nu$ on $\RR^d$, with densities $f_{\mu},f_{\nu}$, respectively, this distance is defined as 
\begin{align}
\label{eq:tv_norm_definition}
    \norm{\mu - \nu}_{TV} &= \half \sup_{g \text{ s.t. } \norm{g}_{\infty}\leq 1} 
    \int g(x) \cdot \SqBrack{f_{\mu}(x) - f_{\nu}(x)} dx 
    &= \half \int \Abs{f_{\mu}(x) - f_{\nu}(x)} dx.
\end{align}
Note that $\int \Abs{f_{\mu}(x) - f_{\nu}(x)} dx$ is in fact the $L_1$ distance, and the equivalence $\norm{\cdot}_{TV} = \half \norm{\cdot}_{L_1}$ is well known, see \cite{cover2012elements}.

For $a \in \Set{0,1}$, let $\mu_a$ be the distribution of $Z$ given $A=a$, i.e. 
 $\mu_a(\cdot) := \Prob{Z = \cdot \cond A=a}$. 
 We denote the distance induced by the representation as $D_{TV}(Z) = \norm{\mu_0 - \mu_1}_{TV}$, and  
 for $\gamma \geq 0$, we say that the representation $Z$ is $\gamma$-fair iff
\begin{equation}
\label{eq:fairness_condition}
     D_{TV}(Z) = \norm{\mu_0 - \mu_1}_{TV} \leq \gamma \spaceo \spaceo \spaceo \spaceo \text{(Fairness Condition)}.
\end{equation} 
Note that \eqref{eq:fairness_condition} is a quantitative relaxation of the ``perfect fairness'' condition in the sense of statistical parity, which requires $Z \indep A$. Specifically, observe that by definition, $Z \indep A$ iff  \eqref{eq:fairness_condition} holds with $\gamma = 0$ (i.e. $\mu_0 = \mu_1$). In addition,  as shown in \cite{madras2018learning}, \eqref{eq:fairness_condition} implies several other common fairness criteria, in particular, bounds on demographic parity and equalized odds metrics for any downstream classifier built on top of $Z$.

Next, we describe the measurement of information loss in $Y$ due to the representation. 
Let $h: \Delta_{\mcY} \rightarrow \RR$ be a continuous and concave function on the set of probability distributions on $\mcY$, $\Delta_{\mcY}$. 
The quantity $h(\Prob{Y \cond X=x})$ will measure the best possible 
prediction accuracy of $Y$ conditioned on $X=x$, for varying $x$. 
As an example, consider the case of binary $Y$, $\mcY = \Set{0,1}$.  
Every point in $\simplxY$ can be written as $(p,1-p)$ for $p\in [0,1]$, and we may choose $h$ to be the optimal binary classification error, 
\begin{equation}
\label{eq:accuraccy_h_definition}
h((1-p,p)) = min(p,1-p).     
\end{equation}
Another possibility it to use the entropy, $h((1-p,p))=p\log p +(1-p) log (1-p)$. The average uncertainty of $Y$ is given by $\Expsubidx{x \sim X}{h(\Prob{Y \cond X=x})}$. Note that this notion does not depend on a particular classifier, but reflects the performance the \emph{best} classifier can possibly achieve (under appropriate cost). 

The goal of fair representation learning is then to find representations $Z$ that for a given $\gamma\geq 0$ satisfy the constraint \eqref{eq:fairness_condition}, and under that constraint minimize the objective $E = E_{\theta}$ given by
\begin{equation}
\label{eq:performance_cost}
    E_{\theta} = \Expsubidx{z \sim Z}{h(\Prob{Y \cond Z = z})}.
\end{equation}
That is, the representation should minimise the optimal $Y$ prediction error (using $Z$) under the fairness constraint. 

The curve that associates to every $0\leq \gamma \leq 1$ the minimum of \eqref{eq:performance_cost} over all representations $Z$ which satisfy \eqref{eq:fairness_condition} with $\gamma$ is referred to as the \emph{Pareto Front} of the Fairness-Performance trade-off. 

In supplementary material Section \ref{sec:mono_loss_reprs} we show that 
for any representation, $\Expsubidx{z \sim Z}{h(\Prob{Y \cond Z = z})} \geq \Expsubidx{x \sim X}{h(\Prob{Y \cond X=x})}$, i.e. representations generally decrease or maintain the performance. 

\subsection{Classification and Factorization}
\label{sec:classification_and_factorisation}
In this section we show that the Pareto front of binary classifiers, with accuracy performance and statistical parity fairness measure, can be computed from the Pareto front of representations with total variation fairness measure. In fact, Lemma \ref{lemma:cls} below states that both Pareto fronts amount to the same curve. 
As discussed in Section \ref{sec:introduction}, this equivalence implies that MIFPO can be used to evaluate fair classifiers, in addition to fair representations.

For a binary classifier $\hat{Y}$ of $Y$, with $(X,A)$ as features. The prediction error is defined as usual by $\epsilon(\hat{Y}) := \Prob{\hat{Y} \neq Y}$. 
The \emph{statistical parity} of $\hat{Y}$ is defined as 
\begin{equation}
    D_{SP}(\hat{Y}) := \Abs{\Prob{\hat{Y} = 1\cond A = 1} - \Prob{\hat{Y} = 1 \cond A = 0} }.
\end{equation}

\begin{lemma}
\label{lemma:cls}
    Let $\hat{Y}$ be a classifier of $Y$, let the representation uncertainty measure be given by \eqref{eq:accuraccy_h_definition}. Then there is a representation 
    given by a random variable $Z$ on a set $\mcZ$ with $\Abs{\mcZ} = 2$, 
    such that 
\begin{align}
    \Expsubidx{z \sim Z}{h(\Prob{Y \cond Z = z})} \leq \epsilon(\hat{Y}) 
    \text{ and } \norm{\mu_0 - \mu_1}_{TV} \leq D_{SP}(\hat{Y}). 
\end{align}    
Conversely, for any given representation $Z$, there is a classifier 
$\hat{Y}$ of $Y$ as a function of $Z$ (and thus of $(X,A)$), such that 
\begin{align}
    \epsilon(\hat{Y}) \leq \Expsubidx{z \sim Z}{h(\Prob{Y \cond Z = z})}   
    \text{ and } D_{SP}(\hat{Y}) \leq \norm{\mu_0 - \mu_1}_{TV}.     
\end{align}
\end{lemma}
The Proof of Lemma \ref{lemma:cls} is presented in Supplementary Material Section \ref{supp:fair_classifiers}. 

We now describe the Factorization result.  
Let $f^*:\mcX \times \mcA \rightarrow \Delta_{\mcY}$ be the Bayes optimal classifier of $Y$ given $X,A$. That is, for every 
$x\in \mcX, a\in \mcA$, $f^*(x,a)$ is the conditional distribution of $Y$ given $x,a$, i.e. $f^*(x,a) = \Prob{Y=\cdot \cond X=x,A=a}$.  Denote by $(X',A)$ a new pair of random variables, taking values in $\Delta_{\mcY}\times A$, given by 
$(X',A) = (f^*(X,A),A)$. 
\begin{lemma}[Factorization]
\label{lem:factorization}
    For any representation $Z$ of $(X,A)$, there is a representation $Z'$ of $(X',A)$, such that 
    \begin{equation}
        \Expsubidx{z' \sim Z'}{h(\Prob{Y \cond Z' = z'})} \leq 
        \Expsubidx{z \sim Z}{h(\Prob{Y \cond Z = z})} 
        \text{ and }
        D_{TV}(Z') \leq D_{TV}(Z).
    \end{equation}
\end{lemma}
In words, for every representation $Z$, we can find a representation $Z'$ that only accesses $(x,a)$ through the value $f^*(x,a)$, and is at least as good in terms of both fairness and performance. Equivalently, this means that any two points $(x_1,a)$ and $(x_2,a)$ with coinciding conditional $Y$ distribution may be treated as identical for the purposes of constructing optimal representations. As a result, to find optimal tradeoffs, we can only consider the representations $Z'$ on the small space $\Delta_{\mcY}\times \mcA$, rather than $Z$ on the much bigger space $\mcX \times \mcA$.

Observations related to Lemma \ref{lem:factorization} were made in the context of classification in \cite{kim2020fact},\cite{xian2023fair}, and \cite{wang2023aleatoric}, which also start from the Bayes optimal classifier. Lemma \ref{lem:factorization} generalizes these observations to representations and to general losses. The proof may be found in Supplementary \ref{sec:factorization_proofs}.

\subsection{The Invertibility Theorem}
\label{sec:opt_reprs_properties}
In this section we define the notion of invertibility for representations, and show that considering invertible representations is sufficient for computing the Pareto front. 

A representation $Z$ on a set $\mcZ$ is \emph{invertible} if for every $z\in \mcZ$ and every $a\in \Set{0,1}$, there is at most one $x\in \mcX$ such that $\Prob{Z=z\cond X=x,A=a}>0$.   In words, a representation is invertible, if any given $z$ can be produced by at most two original features $(x,a)$, and at most one for each value $a$. For $z \in \mcZ$, we say that an 
$(x,a)$ is a \emph{parent} of $z$ if $\Prob{Z=z\cond X=x,A=a}>0$. 

\begin{theorem}
\label{lem:invertibility_lemma}
Let $Z$ be any representation of $(X,A)$ on a set $\mcZ$. 
Then there exists an invertible representation $Z'$ of $(X,A)$, 
on some set $\mcZ'$, such that 
\begin{align}
\label{eq:invertibility_lemma_eq}
    \Expsubidx{z' \sim Z'}{h(\Prob{Y \cond Z' = z'})} 
    \leq \Expsubidx{z \sim Z}{h(\Prob{Y \cond Z = z})} 
    \text{ and } 
    D_{TV}(Z') = D_{TV}(Z).
\end{align}
\end{theorem}
In words, for every representation, we can find an invertible representation of the same data which satisfies at least as good a fairness constraint, and has at least as good performance as the original. In particular, this implies that when one searches for optimal performance representations, it suffices to only search among the invertible ones. 

The proof proceeds by observing that if an atom $z \in \mcZ$ has more than one parent for a fixed $a$, then one can split this atom into two, with each having less parents. However, the details of this construction are somewhat intricate and the full argument can be found in Section \ref{sec:proof_of_invertibility}.

Although in this paper we concentrate on the case of binary sensitive variable, we note that Theorem \ref{lem:invertibility_lemma} may be extended to multi valued attributes, with a similar argument. In that case, invertibility would mean that every $z\in \mcZ$ would still have at most two parents, $u,v$, corresponding to different values $a,a'$ of $A$.


\section{The Model Independent Optimization Problem}
\label{sec:reduction_general}
In this Section we motivate and introduce the MIFPO optimisation problem, and then discuss the full Pareto front computation procedure starting from the raw data. 

\begin{figure}
    \centering
    \subfigure[]{\includegraphics[width=0.48\textwidth, height=3.5cm]{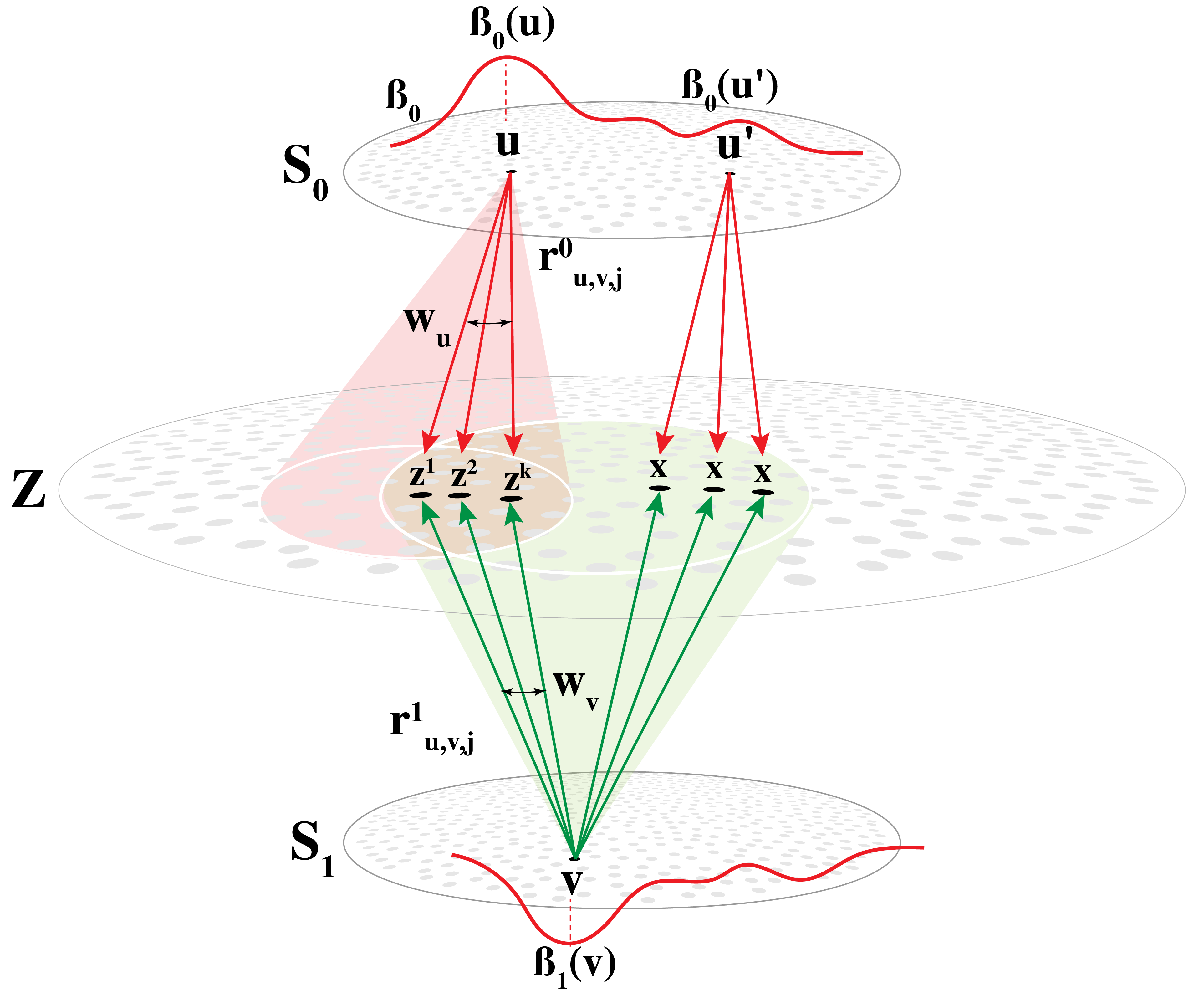}} 
    \subfigure[]{\includegraphics[width=0.48\textwidth, height=3.5cm]{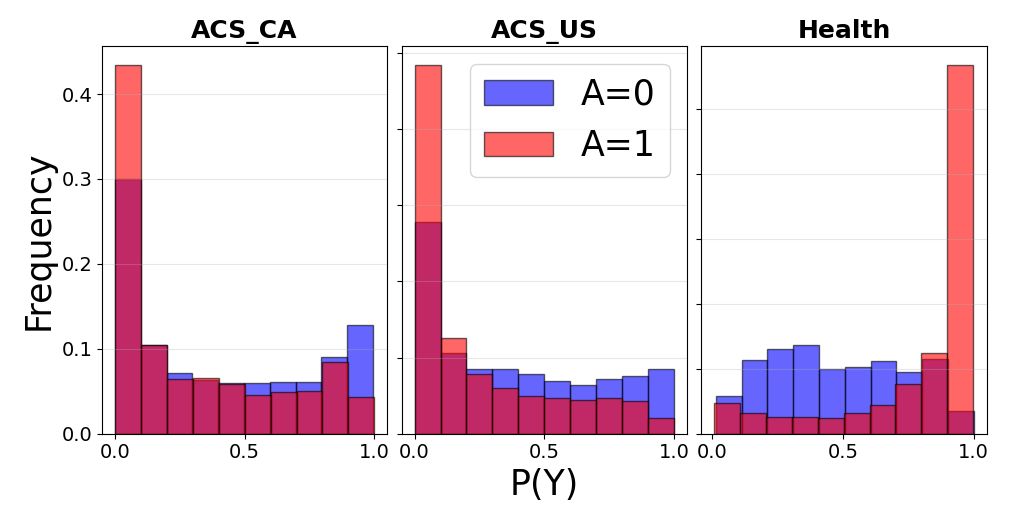}} 
    \caption{(a) The MIFPO Setting (b) Distribution of $P(Y=1|X,A)$ for each group across datasets. }
    \label{fig:technical_details}
\end{figure}

\subsection{MIFPO Definition}
\label{sec:mifpo_construction}
For the purposes of this Section, we assume that the $x$ feature space $\mcX$ is finite. In the next section, Section \ref{sec:practical_construction_details}, we describe how we obtain such finite spaces by using the factorization result and discretizing $\Delta_{\mcY}$.  Note, however, that the full original, possibly high dimensional feature space $\mcX$, is never discretized.  

Write $S_0 = \Set{(x,0) \setsep x \in \mcX} = \mcX \times \Set{0}$, and similarly $S_1 = \mcX \times \Set{1}$, for the two halves of the full feature space, $\mcX \times \mcA = S_0 \cup S_1$.

\textbf{Parameters:} The MIFPO parameters model the data distribution and are as follows: \textbf{(a)} the probability distributions $\beta_0 \in \Delta_{S_0}$ and $\beta_1 \in \Delta_{S_1}$, on $S_0$ and $S_1$ respectively, modeling $\Prob{(X,A) \cond A=0}$ and $\Prob{(X,A) \cond A=1}$ respectively, i.e. the distribution of the data features on each sensitive subgroup.  \textbf{(b)}
The subgroup proportions $\alpha_a = \Prob{A=a}$, and \textbf{(c)} the conditional $Y$ distributions, 
$\rho_u,\rho_v \in \Delta_{\mcY}$, modeling $\rho_u = \Prob{Y=\cdot \cond (X,A) = u}$ when $a=0$ or $\rho_v = \Prob{Y=\cdot \cond (X,A) = v}$ when $a=1$.  

\textbf{Representation Space:} 
Perhaps the first question one can ask when constructing a representation of the data as above is: \emph{How large the representation space should be?} We now answer this question using the theory of Section \ref{sec:full_theory_section}.

Fix an integer $k\geq 2$. The representation space $\mcZ$ will 
be a finite set which can be written as 
\begin{equation}
\label{eq:reprsentation_space_equation}
\mcZ = S_0 \times S_1 \times [k],    
\end{equation}
where $[k]:= \Set{1,2,\ldots,k}$.  
That is, every point $z\in \mcZ$ corresponds to some triplet  $(u,v,j)$, with $u\in S_0, v \in S_1, j\in [k]$. 
To explain this choice, recall that by the Invertibility result, we know that we may consider only invertible representations. In such representations, every point $z\in \mcZ$ is indexed by a pair of parents $(u,v)\in S_0\times S_1$, suggesting that we may index the points by $S_0\times S_1$ to begin with. Next, for a given such pair $(u,v)$, we may ask how many points $z$ should have the same pair $(u,v)$ as their parents? 
In Supplementary Section \ref{sec:uniform_approximation}, we show that using $k$ points for every pair, we can obtain uniform approximation over all representations. That is, given a degree 
of approximation $\eps$, Lemma \ref{lem:uniform_approx} provides a bound on $k$ which is sufficient to obtain such approximation. 
While such a bound would clearly depend on $\eps$, we not that it does note depend on the sizes $\Abs{S_0},\Abs{S_1}$. 
These considerations explain the choice of \eqref{eq:reprsentation_space_equation} as the representation space. We have used $k=5$ in all experiments.

\textbf{Variables:} The variables of the problem model the representation itself. They will be denoted by $r_{u,v,j}^a$ for $(u,v,j) \in \mcZ$ and $a \in \mcA$, and model the probabilities $r_{u,v,j}^a = \Prob{ Z= (u,v,j) \cond (X,A) = s }$, where either $a=0$ and $s=u \in S_0$, or $a=1$ and $s=v \in S_1$ for some $v\in \mcX$. 
That is, for $a=0$, points $u$ transition to $(u,v,j)$ for some $v\in S_1, j \in [k]$,  and similarly for $a=1$, points
$v$ transition to $(u,v,j)$ 
for some $u\in \mcX, j \in [k]$. This notation preserves our  convention that $(u,v,j) \in Z$ has $u$ and $v$ as its only parents. The situation is illustrated in Figure \ref{fig:technical_details}(a). 

Note that the variables represent probabilities, and thus satisfy the following constraints: 
\begin{align}
\label{eq:r_constr1}
  &r_{u,v,j}^a \geq 0,   \spaceo \forall (u,v,j)\in \mcZ, \forall a\in \mcA \\    
  &\sum_{v\in S_1, j\in [k]} r_{u,v,j}^0 = 1    \spaceo \forall u\in S_0  \text{ and }
    \sum_{u\in S_0, j\in [k]} r_{u,v,j}^1 = 1    \spaceo \forall v\in S_1
  \label{eq:r_constr2}
\end{align}

\textbf{Performance Objective and Fairness Constraints:}
With these preparations, we are ready to write the performance cost 
\eqref{eq:performance_cost} in the new notation:
\begin{equation}
\label{eq:Er_perf_cost}
\begin{split}
    E_r = \sum_{z=(u,v,j)}  \SqBrack{\alpha_0 \beta_0(u) r_{u,v,j}^0 + 
    \alpha_1 \beta_1(v) r_{u,v,j}^1} 
     \cdot h\Brack{\frac{\rho_u \alpha_0 \beta_0(u) r_{u,v,j}^0 + \rho_v \alpha_1 \beta_1(v) r_{u,v,j}^1}
    {\alpha_0 \beta_0(u) r_{u,v,j}^0 + \alpha_1 \beta_1(v) r_{u,v,j}^1}}.
\end{split}
\end{equation}
Indeed, observe that due to the structure of our representations, every 
$z$ has two parents, and we have $\Prob{Z=z} = \Brack{\alpha_0 \beta_0(u) r_{u,v,j}^0 + \alpha_1 \beta_1(v) r_{u,v,j}^1}$. Similarly, $\Prob{Y \cond Z=z}$ is computed via 
\begin{equation}
\Prob{Y\cond z} = \sum_{x,a} \Prob{Y\cond x,a,z} \Prob{a,x\cond z} \notag 
= \sum_{x,a}\Prob{Y\cond x,a} \Prob{z\cond a,x} \Prob{a,x} / \Prob{z},
\end{equation}
and substituted inside $h$ to obtain \eqref{eq:Er_perf_cost}. As we show in Supplementary \ref{sec:concavity_of_Er}, the cost \eqref{eq:Er_perf_cost} is a \emph{concave} function of the variables $r$.

We now proceed to discuss the fairness constraint. Recall that we define $\mu_a(z) = \Prob{Z=z\cond A=a}$, for $a\in \Set{0,1}$. For $z=(u,v,j)$ we have then $\mu_a((u,v,j)) =  \beta_a(u) r_{u,v,j}^a$, for $a\in \Set{0,1}$, and we can write 
\begin{equation}
\begin{split}
    D_{TV}(Z) = \norm{\mu_0-\mu_1}_{TV} 
     = \half \sum_z \Abs{\mu_0(z)-\mu_1(z)} 
     = \half \sum_{(u,v,j)} \Abs{\beta_0(u) r_{u,v,j}^0 - \beta_1(v) r_{u,v,j}^1}
\end{split}
\end{equation}
and the Fairness constraint, for a given $\gamma \in [0,1]$, is thus simply 
\begin{equation}
\label{eq:fairness_constraint_construction_notation}
\half \sum_{(u,v,j)} \Abs{\beta_0(u) r_{u,v,j}^0 - \beta_1(v) r_{u,v,j}^1} \leq \gamma.
\end{equation}

We now summarise the full MIFPO problem.
\begin{definition}[MIFPO]  
\label{def:mifpo_def} For a fixed finite ground set $\mcX \times \mcA =S_0\cup S_1$,
the problem parameters are the weight $\alpha_0$, the distributions 
$\beta_0,\beta_1$, on $S_0$ and $S_1$ respectively, and the distributions 
$\rho_x \in \simplxY$ for every $x\in S_0\cup S_1$.
The problem variables are 
$\Set{ r_{u,v,j}^0, r_{u,v,j}^1}_{(u,v,j) \in \mcZ} $ as defined above. We are interested in minimizing the concave function \eqref{eq:Er_perf_cost}, subject to the constraints \eqref{eq:r_constr1}, \eqref{eq:r_constr2}, and \eqref{eq:fairness_constraint_construction_notation}.
\end{definition}

The relationship between MIFPO and the Optimal Transport problem is detailed in Supplementary \ref{sec:mifpo_and_optimal_transport}.

Finally, observe that 
the MIFPO constraints above \emph{linear}, with 
 \eqref{eq:fairness_constraint_construction_notation} being equivalent to two linear inequality constraints. We note that 
 these constraints may be replace by equivalent linear \emph{equality} constraints, via appropriate slack variables, which is more convenient in practice. See Supplementary \ref{sec:linearizing_fairness_constraint} for details. 

\subsection{The Full Algorithm}
\label{sec:practical_construction_details}
In this Section we summarize the full Pareto front computation algorithm, including the estimation of the MIFPO parameters 
$\alpha$,$\beta$ and $\rho$ as discussed above. 

Let $D = \Set{((x_i,a_i),y_i)}_{i\leq N}$ be the dataset, and write $D_a = \Set{((x_i,a_i),y_i) \in D \setsep a_i = a}$, so that $D = D_0 \cup D_1$. 

The algorithm proceeds in the following steps:
\textbf{Step 1:}  we learn the  probability estimators 
$c_0,c_1:\RR^d \rightarrow \Delta_{\mcY}$, separately on $D_0$ and $D_1$. These estimators should approximate the optimal Bayes classifier (Section \ref{sec:classification_and_factorisation}). 
Note that such estimation of probabilities is well studied, and is known as \emph{calibration}, see \cite{niculescu2005predicting}, \cite{kumar2019verified},
\citep{berta2024classifier}.

\textbf{Step 2:} (Discretization and Parameter estimation) For a given integer $L>0$, 
the space $\Delta_{\mcY}$ is discretized into $L$ bins. 
This corresponds to taking the ground sets $S_0,S_1$ in MIFPO to be of size $L$. The data $D_a$ is then mapped into the $L$ bins using $c_a$. The distribution $\beta_a(w)$ is then simply measures the proportion of points $\Set{c_a(x)}_{(x,a)\in D_a}$ that fall into bin $w\leq L$. Finally, for every bin $w$, we choose an arbitrary point inside that bin as the representative distribution, $\rho_w$. The parameters $\alpha_a$ are estimated simply by $\alpha_a = \Abs{D_a}/\Abs{D}$.  Note that for binary $Y$, $\Delta_{\mcY}$ is simply the interval $[0,1]$, which is trivial to discretize.  See Figure \ref{fig:technical_details}(b) for an example of such histograms on real data. We note that one could easily consider more complex discretization schemes, such as clustering, which could be applied efficiently to multi label problems. See Supplementary \ref{sec:computational_complexities} for a discussion.

\textbf{Step 3:} For a given a fairness threshold $\gamma>0$, 
we can now construct the MIFPO instance, Definition \ref{def:mifpo_def}, with $\Abs{S_0} = \Abs{S_1} = L$, the additional approximation parameter $k$, and $\alpha,\beta,\rho$ as discussed above. As discussed in Section \ref{sec:mifpo_construction}, we found it sufficient to use $k=5$ throughout the paper. The MIFPO is then solved using the existing methods, as detailed in Section \ref{sec:experiments}.
The full algorithm is schematically show as Algorithm \ref{alg:param_and_mifpo}, Supplementary \ref{sec:implementation_and_compute_details}.

\section{Experiments}

\label{sec:experiments}
Our approach requires two main computational components: building calibrated classifiers to evaluate $c_a$, and solving the discrete optimization problem described in Section \ref{sec:practical_construction_details}. For the calibrated classifier, we have used XGBoost \citep{chen2015xgboost}, with Isotonic Regression calibration, as implemented in sklearn, \cite{scikit-learn}. 
Next, as discussed in Sections \ref{sec:introduction}, \ref{sec:mifpo_construction}, MIFPO is a concave minimisation problem, under linear constraints. To solve its, we have used the DCCP framework and the associated solver \citep{shen2016disciplined, disciplined_github}, which are based on the combination 
of convex-concave programming (CCP) \cite{lipp2016variations} and disciplined convex programming,  \cite{grant2006disciplined}. 
We note that although local minima are theoretically possible, the above framework is well-established, and the concave structure can  be exploited to allow finding optimal solutions in most practical cases, \citep{shen2016disciplined}.  In particular, our results do not indicate local minima issues. 
However, it may also be worth noting that  MIFPO could in principle be also solved with the classical branch-and-bound methods, \cite{benson1995concave}, which may be slower but do guarantee the global optimum solution.

 Throughout the experiments, we use the missclassification error loss $h$ given by \eqref{eq:accuraccy_h_definition}. Additional implementation details may be found in Supplementary Section \ref{sec:experiments_additional_detail}.

\begin{figure*}
    \centering
    \includegraphics[width=\linewidth, height = 6cm]{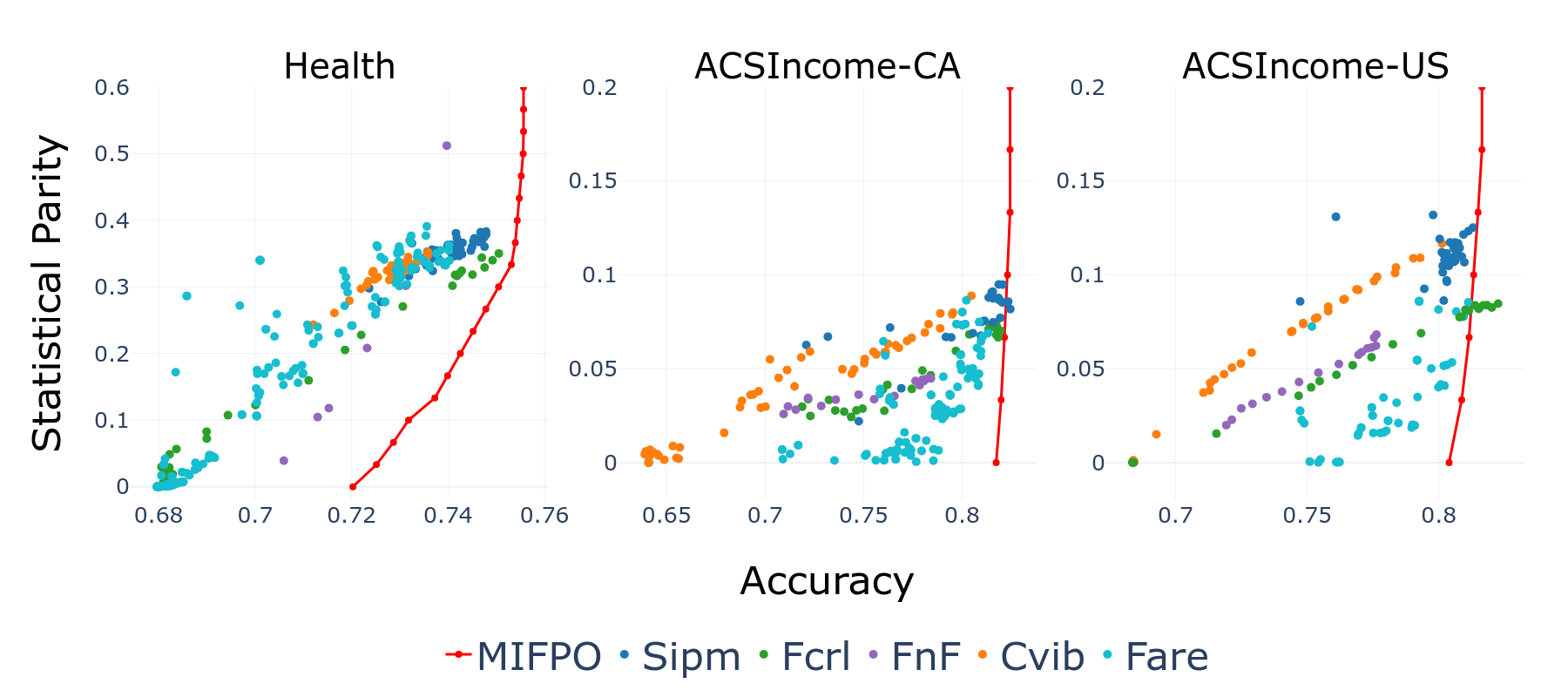}
    \caption{Comparison of fairness-accuracy trade-offs across three benchmark datasets: Health (left), ACSIncome-CA (middle), and ACSIncome-US (right). MIFPO's Pareto front is represented as a solid line with markers. The horizontal axis represents the fairness constraint (statistical parity distance), while the vertical axis shows prediction accuracy.}
    \label{fig:ACC_DP}
\end{figure*}

Our experimental validation of MIFPO encompasses three standard fairness benchmarks: the Health dataset alongside two variants of ACSIncome—one restricted to California (ACSIncome-CA) and another spanning the entire United States (ACSIncome-US). In Figure \ref{fig:ACC_DP} we evaluate MIFPO against five state-of-the-art fair representation techniques: CVIB \citep{moyer2019invariantrepresentationsadversarialtraining}, FCRL \citep{gupta2021controllableguaranteesfairoutcomes}, FNF \citep{fnf_github}, sIPM \citep{kim2023learningfairrepresentationparametric}, and Fare \citep{jovanovic2023fare}. Each competitive approach was tuned across diverse hyperparameter settings to generate a spectrum of representations balancing fairness and accuracy. Moreover, we evaluated MIFPO against several fair-classification methods on multiple datasets as presented in the Supplementary Section \ref{sec:experiments_additional_detail}. 

The empirical results presented in Figure \ref{fig:ACC_DP} demonstrate MIFPO's effectiveness relative to prior approaches. MIFPO consistently achieves performance equal to or superior than the baseline methods across almost all operating points. Furthermore, MIFPO provides a significant methodological advantage through its ability to characterize the complete Pareto frontier. In the figure, MIFPO's performance is visualized as a solid line with points that trace the entire Pareto front, while competing algorithms are represented as individual points corresponding to different hyperparameter configurations. 

\subsection{Implementation}
All evaluations can be found at \url{https://github.com/bp6725/Efficient-Fair-Pareto-Paper}. The MIFPO algorithm's source code is available in the \url{https://github.com/bp6725/FairPareto} repository. The algorithm is also implemented as the "FairPareto" Python package on PyPI, which provides a scikit-learn compatible API for computing optimal fairness-performance Pareto fronts. The package supports two usage modes: a tabular mode with automatic classifier training and calibration given a sensitive attribute column, and a second mode for any data type (images, text) where users provide pre-trained classifiers for each sensitive group. This enables researchers to benchmark their fair classification methods against theoretical optimality with minimal code and make informed decisions about fairness-performance trade-offs. The package is open-source, available on PyPI, and includes comprehensive documentation with examples for both tabular and image data.

\section{Conclusions, Limitations, And Future Work}
\label{sec:conclusion}
In this paper we have introduced new fundamental properties of optimal fair representations. In particular, these are the first theoretical results that allow approximation of the Pareto front for arbitrary concave performance measures.
We have used these results to develop a model independent procedure for the computation of Fairness-Performance Pareto front from data, demonstrated the procedure on real datasets, and have shown that it may be used as a benchmark for other representation learning algorithms. 

We now discuss limitations and a few possible directions for future work. This work primarily concentrated on binary sensitive attribute $A$ and binary $Y$, with the aim to develop the underlying new principles in the simplest case first. 
As discussed earlier (Sections \ref{sec:introduction}, \ref{sec:computational_complexities}), the multi-label case may be treated by more elaborate discretizations. We also noted that the Invertibility Theorem holds for multi valued sensitive attributes as well, which allows to extend the approximation analysis to that case too. Both of these steps, however, would increase the MIFPO problem size. On the other hand, it is also worth noting that this size does not depend directly neither on the feature dimension $d$, nor on the sample size $N$ and thus the problem scales well in that sense. 

In view of these observations, we believe it would be of interest to study the following question on the true complexity of the tradeoff evaluation: Suppose we are given access to the Bayes optimal classifier of the data, $f^*$. This encapsulates, in a sense, most of the ``continuous'' information of the problem.  Then, how scalable can Pareto estimation methods be made theoretically, in terms of $\Abs{\mcY},\Abs{\mcA}$, while still maintaining controllable approximation bounds?

\begin{ack}

This work has received funding from the European Union’s Horizon Europe research and innovation
programme under grant agreement No. 101070568. 

\end{ack}

\bibliography{refs.bib}
\bibliographystyle{apalike}
\newpage
\section*{NeurIPS Paper Checklist}

\begin{enumerate}

\item {\bf Claims}
    \item[] Question: Do the main claims made in the abstract and introduction accurately reflect the paper's contributions and scope?
    \item[] Answer: \answerYes{} 
    \item[] Justification: All claims are supported in the paper.

\item {\bf Limitations}
    \item[] Question: Does the paper discuss the limitations of the work performed by the authors?
    \item[] Answer: \answerYes{} 
    \item[] Justification: A discussion may be found in Section \ref{sec:conclusion}, as well as in Section \ref{sec:literature} in context of comparison to other methods. 

\item {\bf Theory assumptions and proofs}
    \item[] Question: For each theoretical result, does the paper provide the full set of assumptions and a complete (and correct) proof?
    \item[] Answer: \answerYes{} 
    \item[] Justification: All results are formally stated and have proofs in the Supplementary. 

    \item {\bf Experimental result reproducibility}
    \item[] Question: Does the paper fully disclose all the information needed to reproduce the main experimental results of the paper to the extent that it affects the main claims and/or conclusions of the paper (regardless of whether the code and data are provided or not)?
    \item[] Answer: \answerYes{} 
    \item[] Justification: All information on reproducibility is included in Section \ref{sec:experiments} and corresponding Supplementary sections. The full code will be published with the final version of the paper. 

\item {\bf Open access to data and code}
    \item[] Question: Does the paper provide open access to the data and code, with sufficient instructions to faithfully reproduce the main experimental results, as described in supplemental material?
    \item[] Answer: \answerNo{} 
    \item[] Justification: While we describe all reproducibility details, the full code will be published with the final version of the paper. 
    The datasets used are publicly available. 

\item {\bf Experimental setting/details}
    \item[] Question: Does the paper specify all the training and test details (e.g., data splits, hyperparameters, how they were chosen, type of optimizer, etc.) necessary to understand the results?
    \item[] Answer: \answerYes{} 
    \item[] Justification: We supply all the necessary details for the reproduction of the results. Furthermore, the full details will be included as part of the published code. 

\item {\bf Experiment statistical significance}
    \item[] Question: Does the paper report error bars suitably and correctly defined or other appropriate information about the statistical significance of the experiments?
    \item[] Answer: \answerYes{} 
    \item[] Justification: The methods discussed in this paper produce identical results over multiple runs (std <= 1e-5). 
    We note that for some benchmark methods we have used existing evaluation results; however, the authors of those evaluations claim to have conducted multiple runs to validate stability, and the results presented on the graphs are from different hyperparameters as an integral part of the analysis.
    

\item {\bf Experiments compute resources}
    \item[] Question: For each experiment, does the paper provide sufficient information on the computer resources (type of compute workers, memory, time of execution) needed to reproduce the experiments?
    \item[] Answer: \answerYes{} 
    \item[] Justification: Yes, relevant information is provided in the Supplementary. We note that all experiments of this paper together run in under day on a standard desktop. 
    
\item {\bf Code of ethics}
    \item[] Question: Does the research conducted in the paper conform, in every respect, with the NeurIPS Code of Ethics \url{https://neurips.cc/public/EthicsGuidelines}?
    \item[] Answer: \answerYes{} 
    \item[] Justification: We have examined the Code of Ethics and verified that the research conforms with the code. 

\item {\bf Broader impacts}
    \item[] Question: Does the paper discuss both potential positive societal impacts and negative societal impacts of the work performed?
    \item[] Answer: \answerNA{} 
    \item[] Justification: This work is in Fairness, with obvious societal impact implications. However, while our approach is algorithmically new and can provide practically and empirically better evaluation of fairness-performance tradeoffs, the general societal implications of such tradeoffs are well understood in the field and we believe require no special new consideration in this particular work. 
    
    
\item {\bf Safeguards}
    \item[] Question: Does the paper describe safeguards that have been put in place for responsible release of data or models that have a high risk for misuse (e.g., pretrained language models, image generators, or scraped datasets)?
    \item[] Answer: \answerNA{} 
    \item[] Justification: We do not see reasonable scenarios in which release of our code may be risky. 

\item {\bf Licenses for existing assets}
    \item[] Question: Are the creators or original owners of assets (e.g., code, data, models), used in the paper, properly credited and are the license and terms of use explicitly mentioned and properly respected?
    \item[] Answer: \answerYes{} 
    \item[] Justification: All assets are precisely referenced. 

\item {\bf New assets}
    \item[] Question: Are new assets introduced in the paper well documented and is the documentation provided alongside the assets?
    \item[] Answer: \answerNA{} 
    \item[] Justification: We do not release new assets this time. As mentioned earlier, the implementation of the methods in this paper will be released with the final version of the paper. 

\item {\bf Crowdsourcing and research with human subjects}
    \item[] Question: For crowdsourcing experiments and research with human subjects, does the paper include the full text of instructions given to participants and screenshots, if applicable, as well as details about compensation (if any)? 
    \item[] Answer: \answerNA{} 
    \item[] Justification: The paper does not involve crowdsourcing or reesearch with humans. 

\item {\bf Institutional review board (IRB) approvals or equivalent for research with human subjects}
    \item[] Question: Does the paper describe potential risks incurred by study participants, whether such risks were disclosed to the subjects, and whether Institutional Review Board (IRB) approvals (or an equivalent approval/review based on the requirements of your country or institution) were obtained?
    \item[] Answer: \answerNA{} 
    \item[] Justification: The paper does not involve crowdsourcing nor research with human subjects.

\item {\bf Declaration of LLM usage}
    \item[] Question: Does the paper describe the usage of LLMs if it is an important, original, or non-standard component of the core methods in this research? Note that if the LLM is used only for writing, editing, or formatting purposes and does not impact the core methodology, scientific rigorousness, or originality of the research, declaration is not required.
    \item[] Answer: \answerNA{} 
    \item[] Justification: This research does not involve LLMs as any important, original, or non-standard components.

\end{enumerate}

\newpage
\appendix

\section*{Supplementary Material} 
\addcontentsline{toc}{section}{Supplementary Material} 

\startcontents[sections]
\printcontents[sections]{l}{1}{\setcounter{tocdepth}{2}}

\section{Monotonicity of Loss Under Representations}
\label{sec:mono_loss_reprs}
As discussed in Section \ref{sec:problem_setting}, we observe that representations can not increase the performance of the classifier (i.e decrease the loss). 
\begin{lemma} 
\label{lem:data_processing}
For every $(Y,X,A)$, every representation $Z$ as above, and concave $h$,
\begin{equation}
\Expsubidx{z \sim Z}{h(\Prob{Y \cond Z = z})} 
\geq  \Expsubidx{(x,a) \sim (X,A)} {h(\Prob{Y \cond X= x, A=a})}.            
\end{equation}
\end{lemma}
Note that the right hand-side above can be considered a ``trivial" representation, $Z=(X,A)$. 

In what follows, to simplify the notation we use expressions of the 
form $\Prob{x,a \cond z}$ to denote the formal expressions $\Prob{X=x,A=a \cond Z=z}$, whenever the precise interpretation is clear from context.

\begin{proof}
For every value $y \in \mcY$, we have 
\begin{flalign}
\label{eq:lem_dat_proc_1}
\Prob{Y=y \cond Z=z} &= 
\sum_a \int dx \spaceo \Prob{Y=y \cond x,a,z} \frac{\partial \Prob{x \cond a,z}}{\partial x} \Prob{a \cond z} \\ 
\label{eq:lem_dat_proc_2}
&= \sum_a \int dx \spaceo \Prob{Y=y \cond x,a} \frac{\partial \Prob{x \cond a,z}}{\partial x} \Prob{a \cond z} \\ 
\label{eq:lem_dat_proc_3}
&= \Expsubidx{(x,a) \sim (X,A)\cond Z=z}{\Prob{Y=y \cond x,a}}.
\end{flalign}
Here, on line \eqref{eq:lem_dat_proc_1}, $\frac{\partial \Prob{x \cond a,z}}{\partial x}$ is the density of $\Prob{x \cond a,z}$ with respect to $dx$. 
Crucially, the transition from \eqref{eq:lem_dat_proc_1} to \eqref{eq:lem_dat_proc_2} is using the property \eqref{eq:representation_y_indep_condition}. 
The transition \eqref{eq:lem_dat_proc_2} to \eqref{eq:lem_dat_proc_3} is a change of notation. 
Using \eqref{eq:lem_dat_proc_3} and the concavity of $h$, we obtain 
\begin{flalign}
\Expsubidx{z \sim Z}{h(\Prob{Y \cond Z = z})} &\geq 
\Expsubidx{z \sim Z}{\Expsubidx{(x,a) \sim (X,A)\cond Z=z}{h\Brack{\Prob{Y \cond x,a}}}} \\
&= \Expsubidx{(x,a) \sim (X,A)} {h(\Prob{Y \cond X= x, A=a})}.            
\end{flalign}
\end{proof}

\section{MIFPO and Optimal Transport}
\label{sec:mifpo_and_optimal_transport}
In this Section we discuss the relation between the MIFPO minimisation problem, Definition \ref{def:mifpo_def}, and the problem of Optimal Transport (OT).  General background on OT may be found in \cite{peyre2019computational}.
We discuss the similarity between OT and the minimisation of \eqref{eq:Er_perf_cost} under the constraint \eqref{eq:fairness_constraint_construction_notation} with  $\gamma=0$, i.e. the perfectly fair case. 
In this case, \eqref{eq:fairness_constraint_construction_notation} is equivalent to the condition $\beta_0(u)r_{u,v,j} = \beta_1(v)r_{v,u,j}$, 
for all $(u,v,j) \in \mcZ$. Next, note that thus in is case the 
expression for $\Prob{Y\cond Z=(u,v,j)}$ is
\begin{equation}
 \frac{\rho_u \alpha_0 \beta_0(u) r_{u,v,j}  + \rho_v \alpha_1 \beta_1(v) r_{v,u,j} }{\alpha_0 \beta_0(u) r_{u,v,j} + 
    \alpha_1 \beta_1(v) r_{v,u,j}} = \frac{\rho_u \alpha_0 + \rho_v \alpha_1}{\alpha_0 + \alpha_1},    
\end{equation}
and \emph{this is independent} of the variables $r$! Therefore we can write 
the cost \eqref{eq:Er_perf_cost} as 
\begin{equation}
\label{eq:pseudo_OT_cost}
    \sum_{u,v} (\alpha_0 + \alpha_1) h\Brack{\frac{\rho_u \alpha_0 + \rho_v \alpha_1}{\alpha_0 + \alpha_1}}
    \half \SqBrack{\sum_{j\leq k} \beta_0(u)r_{u,v,j} + \beta_1(v)r_{v,u,j}}
    .
\end{equation}
Note further that for fixed $u,v$, the different $j$'s in this expression play similar roles and could be effectively merged as a single point. 

The cost \eqref{eq:pseudo_OT_cost} has several similarities with OT.  First, in both problems we have 
two sides, $S_0$ and $S_1$, and we have a certain fixed loss associated with "matching" $u$ and $v$. In case of \eqref{eq:pseudo_OT_cost}, this loss is $(\alpha_0 + \alpha_1) h\Brack{\frac{\rho_u \alpha_0 + \rho_v \alpha_1}{\alpha_0 + \alpha_1}}$, which describes the information loss 
incurred by colliding $u$ and $v$ in the representation. 
And second, similarly to OT, \eqref{eq:pseudo_OT_cost} it is \emph{linear} in the variables $r$. Linear programs are conceptually considerably simpler than minimisation of the concave objective \eqref{eq:Er_perf_cost}.

\section{Proof Of Theorem \ref{lem:invertibility_lemma}}
\label{sec:proof_of_invertibility}
In this Section we prove Theorem \ref{lem:invertibility_lemma}.

To keep the notation and the main argument concise, we prove the result under the assumption that $(X,A)$ is finitely supported. Since no assumptions are made on the cardinalities of the supports, the general measurable case follows by standard approximation arguments.

We now introduce the additional notation necessary for the proof. Let $S_0,S_1$ be finite disjoint sets, where $S_a$ represents the values of $(X,A)$ when $A=a$, for $a\in \Set{0,1}$. Denote $S=S_0 \cup S_1$. We are assuming that there is a probability distribution $\zeta$ on $S$, and $A$ is the random 
variable $A = \Ind{s\in S_1}$.  
$X$ is defined as taking the values $s\in S$, with 
$\Prob{X=s} = \zeta(s)$. 
Further, the variable $Y$ is defined to take values in a finite set $\mcY$, and for every $s\in S$, its conditional distribution is given by $\rho_s\in \Delta_{\mcY}$. That is, $\Prob{Y=y \cond X=s,A=a} = \rho_s(y) = \rho_{s,a}(y)$. \footnote{Note that there is a slight redundancy in the notation 
$\Prob{Y=y \cond X=s,A=a}$ here, since $a$ is determined by $s$. However, to retain compatibility with the standard notation, literature, we specify them both. 
This is similar to the continuous situation, in which although $A$ is technically part of the features, $X$ and $A$ are specified separately.}
This completes the description of the data model. 

For $a\in \Set{0,1}$ we denote $\alpha_a = \Prob{A=a} = \zeta(S_a)$, and 
$\beta_a(s) = \Prob{X=s \cond A=a}$.  Observe that $\beta_a(s) = 0$ if $s \notin S_a$, and $\beta_a(s) = \zeta(s) / \zeta(S_a)$ if $s \in S_a$.

We now describe the representation.
The representation will take values in a finite set $\mcZ$. 
For every $s\in S$ and $z \in \mcZ$, let $T_a(z,s) = \Prob{Z=z \cond X=s, A=a}$  be the conditional probability of representing $s$ as $z$. $T_a$ are sometimes referred as the \emph{transition kernels} of the representation. 
For fixed $(X,A)$, the $T_a$'s fully define the distribution of the representation $Z$ and we shall refer to the representation as $T$ or as $Z$ in interchangeably. 
Finally, for $a\in \Set{0,1}$ denote 
\begin{equation}
\mu_a(z) = T_a \beta_a = \Prob{Z=z \cond A=a} = \sum_{s\in S_a} \beta_a(s) T_a(z,s).    
\end{equation}

With the new notation, a representation $T$ is \emph{invertible} if for every $z\in \mcZ$ and every $a\in \Set{0,1}$, there is at most one $s \in S_a$ such that $T_a(z,s)>0$.   In words, a representation is invertible, if any given $z$ can be produced by at most two original features $s$, and at most one in each of $S_0$ and $S_1$. 

Given a representation $T$ and $z \in \mcZ$, we say that an $s\in S$ is a 
\emph{parent} of $z$ if $T_a(z,s)>0$ for the appropriate $a$.

\begin{proof}[Proof Of Theorem \ref{lem:invertibility_lemma}]
Assume $T$ is not invertible. Then there is a $z \in \mcZ$ which has at least two parents in either $S_0$ or $S_1$. Assume without loss of generality that $z$ has two parents in $S_0$. 
Let
\begin{equation}
    U = \Set{ s\in S_0 \setsep T_0(z,s)> 0}, \spaceo
    V = \Set{ s\in S_1 \setsep T_1(z,s)> 0}
\end{equation}
be the sets of parents of $z$ in $S_0$ and $S_1$ respectively. 
Chose a point $x \in U$, and denote by $U^r = U \setminus \Set{x}$ the remainder of $U$. By assumption we have $\Abs{U^r}\geq 1$. 
We also assume that $\Abs{V}>0$. The easier case $\Abs{V}=0$ will be discussed later. 

Now, we construct a new representation, $T'$. The range of $T'$ will be $\mcZ' = \mcZ \setminus{z} \cup \Set{z',z''}$. That is, we remove $z$ and add two new points. 
Denote 
\begin{equation}
    \kappa = \Prob{x\cond z,a=0} = \frac{\beta_0(x)T_0(z,x)}{\sum_{s\in U}\beta_0(s)T_0(z,s)}.
\end{equation}

Then $T'$ is defined as follows: 
\begin{equation}
\label{eq:Tprime_def}
\begin{cases}
    T'_a(h,s) = T_a(h,s) & \text{for all $a\in \Set{0,1}$, all $s\in S$ and all $h\in \mcZ \setminus \Set{z}$ } \\ 
    T'_0(z',x) = T_0(z,x) & \\ 
    T'_0(z'',u) = T_0(z,u) & \text{ for all $u\in U^r$}\\ 
    T'_1(z',v) = \kappa T_1(z,v) & \text{ for all $v\in V$}\\ 
    T'_1(z'',v) = (1-\kappa) T_1(z,v) & \text{ for all $v\in V$}.
\end{cases}    
\end{equation}
All values of $T'$ that were not explicitly defined in \eqref{eq:Tprime_def} are set to $0$. In words, on the side of $S_0$, we move all the parents of $z$ except $x$ to be the parents of $z''$, while $z'$ will have a single parent, $x$. On the $S_1$ side, both $z'$ and $z''$ will have the same parents as $z$, with transitions multiplied by $\kappa$ and $1-\kappa$ respectively. The multiplication by $\kappa$ is crucial for showing  both inequlaities in \eqref{eq:invertibility_lemma_eq}.

Note that  $T'$ can be though of as splitting $z$ into $z'$ and $z''$, such that $z'$ has one parent on the $S_0$ side, and $z''$ has strictly less parents than $z$ had. Once we show that $T'$ satisfies \eqref{eq:invertibility_lemma_eq}, it is clear that by induction we can continue splitting $T'$ until we arrive at an invertible representation which can no longer be split, thus proving the Lemma. 

In order to show \eqref{eq:invertibility_lemma_eq} for $T'$, we will sequentially show the following claims:
\begin{flalign}
    \label{eq:inv_lem_claims1}
    &\Prob{z} = \Prob{z'} + \Prob{z''} \text{ and } \Prob{z'} = \kappa \Prob{z}    \\ 
    \label{eq:inv_lem_claims2}
    &\Prob{a\cond z} = \Prob{a\cond z'} = \Prob{a\cond z''} \text{ for $a\in \Set{0,1}$} \\ 
    \label{eq:inv_lem_claims3}
    &\begin{cases}
    \text{ for $a=1$} &      
        \Prob{s\cond z,a } = \Prob{s\cond z',a } = \Prob{s\cond z'',a } \spaceo \forall s\in S  \\ 
    \text{ for $a=0$, $s\in U^r$} & 
        \begin{cases}
        \Prob{x\cond z',a} = 1 & \Prob{x\cond z'',a} = 0    \\ 
        \Prob{s\cond z',a} = 0 & 
        \Prob{s\cond z'',a} = (1-\kappa)^{-1} \Prob{s\cond z,a} 
        \end{cases}
    \end{cases} \\ 
    \label{eq:inv_lem_claims4}
    &\Prob{Y \cond z} = \kappa \Prob{Y \cond z'} + (1-\kappa)\Prob{Y \cond z''}\\
    \label{eq:inv_lem_claims5}
    &\Prob{z}h(\Prob{Y \cond z}) \geq 
    \Prob{z'}h(\Prob{Y \cond z'}) + \Prob{z''}h(\Prob{Y \cond z''}) \\ 
    \label{eq:inv_lem_claims6}
    &\Abs{\mu_0(z) - \mu_1(z)} = \Abs{\mu'_0(z') - \mu'_1(z')} + \Abs{\mu'_0(z'') - \mu'_1(z'')}.
\end{flalign}
Here the probabilities involving $z',z''$ refer to the representation $T'$.
Observe that the left hand side of \eqref{eq:inv_lem_claims5} is the contribution of $z$ to the performance cost $\Expsubidx{t\sim Z}{ h(\Prob{Y\cond Z=t})}$ of $T$, while the right hand side of \eqref{eq:inv_lem_claims5} is the contribution of $z',z''$ to the performance cost of $T'$. Since all other elements $t\in \mcZ$ have identical contributions, this shows the first inequality in \eqref{eq:invertibility_lemma_eq}.  
Similarly, 
recall that 
\begin{equation}
\norm{\mu_0 - \mu_1}_{TV} =  \half \sum_{t\in \mcZ} \Abs{\mu_0(t) - \mu_1(t)},
\end{equation}
and thus the left hand side of \eqref{eq:inv_lem_claims6} is the contribution of $z$ to $\norm{\mu_0 - \mu_1}_{TV}$, with the right hand side being the contribution of $z',z''$ to $\norm{\mu'_0 - \mu'_1}_{TV}$, therefore yielding the claim $\norm{\mu'_0 - \mu'_1}_{TV} = \norm{\mu_0 - \mu_1}_{TV}$.

\textbf{Claim \eqref{eq:inv_lem_claims1}:} By definition, 
\begin{flalign}
    \Prob{z'} &= \alpha_0 \beta_0(x) T_0'(z',x) + \alpha_1 \sum_{s\in V} \beta_1(s)T'_1(z',s) \\ 
    \label{eq:inv_lem_claim1_derivation1}
    &= \alpha_0 \beta_0(x) T_0(z,x) + \kappa \alpha_1 \sum_{s\in V} \beta_1(s)T_1(z,s) \\
    &= \kappa \alpha_0 \sum_{s\in U} \beta_0(s) T_0(z,s) + \kappa \alpha_1 \sum_{s\in V} \beta_1(s)T_1(z,s) \\ 
    &= \kappa \Prob{z}.
\end{flalign}
Similarly, by definition we have 
\begin{flalign}
    \Prob{z''} &= 
     \alpha_0 \sum_{s\in U^r} \beta_0(s) T_0(z,s) + (1-\kappa) \kappa \alpha_1 \sum_{s\in V} \beta_1(s)T_1(z,s),
\end{flalign}
and summing this with \eqref{eq:inv_lem_claim1_derivation1}, we obtain 
$\Prob{z} = \Prob{z'} + \Prob{z''}$.

\textbf{Claim \eqref{eq:inv_lem_claims2}:} 
Note that it is sufficient to prove the claim for $a=0$ since the probabilities sum to 1.  Write 
\begin{flalign}
    \Prob{a=0 \cond z} &= \frac{\Prob{a=0,z}}{\Prob{z}}    \\ 
    &= \frac{\alpha_0 \sum_{s\in U} \beta_0(s)T_0(z,s)}{\Prob{z}}    \\ 
    &= \frac{\kappa \cdot \alpha_0 \sum_{s\in U} \beta_0(s)T_0(z,s)}{\kappa \cdot \Prob{z}}    \\     
    &= \frac{ \alpha_0 \beta_0(x)T_0(z,x)}{\Prob{z'}}      \\ 
    &= \Prob{a=0 \cond z'}.
\end{flalign}
Similarly, 
\begin{flalign}
    \Prob{a=0 \cond z} &= \frac{\Prob{a=0,z}}{\Prob{z}}    \\ 
    &= \frac{\alpha_0 \sum_{s\in U} \beta_0(s)T_0(z,s)}{\Prob{z}}    \\ 
    &= \frac{(1-\kappa) \cdot \alpha_0 \sum_{s\in U} \beta_0(s)T_0(z,s)}{(1-\kappa) \cdot \Prob{z}}    \\     
    &= \frac{ \alpha_0 \sum_{s\in U^r} \beta_0(s)T_0(z,s)}{\Prob{z''}}   \\ 
    &= \Prob{a=0 \cond z''}.
\end{flalign}

\textbf{Claim \eqref{eq:inv_lem_claims3}:} 
For $a=1$, let us show $\Prob{s\cond z,a } = \Prob{s\cond z',a }$. 
\begin{flalign}
    \Prob{s\cond z,a=1 } &= 
    \frac{\alpha_1 \beta_1(s) T_1(z,s)}{\alpha_1 \sum_{s'\in V}\beta_1(s') T_1(z,s')} \\
    &= \frac{\kappa \alpha_1 \beta_1(s) T_1(z,s)}{\kappa \alpha_1 \sum_{s'\in V}\beta_1(s') T_1(z,s')} \\ 
    &= \frac{ \alpha_1 \beta_1(s) T'_1(z,s)}{\alpha_1 \sum_{s'\in V}\beta_1(s') T'_1(z,s')} \\ 
    &= \Prob{s\cond z',a=1 }.
\end{flalign}
The statement $\Prob{s\cond z,a } = \Prob{s\cond z'',a }$ is shown similarly. Next, for $a=0$, we have $\Prob{x\cond z',a} = 1$  and $\Prob{x\cond z'',a} = 0$ by the definition of the coupling $T'$. 
Moreover, for $s\in U^r$, $\Prob{s\cond z',a} = 0$ also follows by the definition of $T'$.  
Finally, write 
\begin{flalign}
    \Prob{s\cond z'',a=0} &=  \frac{\alpha_0 \beta_0(s) T'_0(z'',s)}{\sum_{s\in U^r} \alpha_0 \beta_0(s') T'_0(z'',s')} \\
    &=  \frac{\alpha_0 \beta_0(s) T_0(z,s)}{\sum_{s\in U^r} \alpha_0 \beta_0(s') T_0(z,s')} \\
    &=  \frac{\alpha_0 \beta_0(s) T_0(z,s)}{(1-\kappa)\sum_{s\in U} \alpha_0 \beta_0(s') T_0(z,s')} \\    
    &= (1-\kappa)^{-1} \Prob{s\cond z,a=0}.
\end{flalign}

\textbf{Claim \eqref{eq:inv_lem_claims4}:} 
We first observe that for any representation (and any z),
\begin{flalign}
\label{eq:in_lem_cl4_l1}
    \Prob{Y =y \cond Z=z} &= \frac{\sum_{s,a} \Prob{Y=y, s,a,z}}{\Prob{z}} \\
\label{eq:in_lem_cl4_l2}
    &= \frac{\sum_{s,a} \Prob{Y=y\cond s,a} \Prob{s,a,z}}{\Prob{z}} \\
\label{eq:in_lem_cl4_l3}
    &= \sum_a \Prob{a \cond z} \SqBrack{\sum_{s\in S_a} \Prob{Y=y\cond s,a} \Prob{s \cond a,z}},
\end{flalign}
where we have used the property \eqref{eq:representation_y_indep_condition} for the transition \eqref{eq:in_lem_cl4_l1}-\eqref{eq:in_lem_cl4_l2}.
Now, using \eqref{eq:inv_lem_claims3}, for $a=1$ we have 
\begin{flalign}
    &\sum_{s\in S_1} \Prob{Y=y\cond s,a=1} \Prob{s \cond a=1,z} = 
    \sum_{s\in S_1} \Prob{Y=y\cond s,a=1} \Prob{s \cond a=1,z'} \notag
    \\
    &=
\label{eq:in_lem_cl4_ll3}
    \sum_{s\in S_1} \Prob{Y=y\cond s,a=1} \Prob{s \cond a=1,z''}.
\end{flalign}
For $a=0$, we have for $z'$ using \eqref{eq:inv_lem_claims3}:
\begin{flalign}
\label{eq:in_lem_cl4_lll3}
    \sum_{s\in S_0} \Prob{Y=y\cond s,a=1} \Prob{s \cond a=1,z'} = 
    \Prob{Y=y\cond x,a=1}.
\end{flalign}
For $a=0$ and $z''$ we have
\begin{flalign}
    \sum_{s\in S_0} \Prob{Y=y\cond s,a=1} \Prob{s \cond a=1,z''} &= 
    \sum_{s\in U^r} \Prob{Y=y\cond s,a=1} \Prob{s \cond a=1,z''} \\ 
\label{eq:in_lem_cl4_llll3}
    &= 
    (1-\kappa)^{-1} \sum_{s\in U^r} \Prob{Y=y\cond s,a=1} \Prob{s \cond a=1,z} 
\end{flalign}
where we have used \eqref{eq:inv_lem_claims3} again on the last line. 

Combining \eqref{eq:in_lem_cl4_ll3},\eqref{eq:in_lem_cl4_lll3},\eqref{eq:in_lem_cl4_llll3}, and using \eqref{eq:inv_lem_claims2} and the general expression \eqref{eq:in_lem_cl4_l3}, we obtain the claim \eqref{eq:inv_lem_claims4}.

\textbf{Claim \eqref{eq:inv_lem_claims5}:}  This follows immediately from 
\eqref{eq:inv_lem_claims4} by using \eqref{eq:inv_lem_claims1} and the concavity of $h$.

\textbf{Claim \eqref{eq:inv_lem_claims6}:} 
By definition, for every representation, $\mu_a(z) = \Prob{Z=z\cond a} = \frac{\Prob{z\cond a}\Prob{z}}{\Prob{a}}$. 
Thus, using \eqref{eq:inv_lem_claims1},\eqref{eq:inv_lem_claims2} we have
for $a\in \Set{0,1}$,
\begin{equation}
    \mu_a(z') = \kappa \mu_a(z) \text{ and } \mu_a(z'') = (1-\kappa) \mu_a(z),
\end{equation}
which in turn yields \eqref{eq:inv_lem_claims6}.

It remains only to recall that we have derived \eqref{eq:inv_lem_claims5},\eqref{eq:inv_lem_claims6} under the assumption that $\Abs{V}>0$. That is, we assumed that the point $z$ which fails invertability on $S_0$ has some parents in $S_1$. The case when $\Abs{V}=0$, i.e. there are no parents in $S_1$ can be treated using a similar argument, but is much simpler. Indeed, in this case one can simply split $z$ into $z'$ and $z''$ and splitting the $S_0$ weight between them as before, without the need to carefully balance the interaction of probabilities with $S_1$ via $\kappa$. 

\end{proof}

\section{Uniform Approximation and Two Point Representations}
\label{sec:uniform_approximation}

As discussed in Sections \ref{sec:introduction},\ref{sec:mifpo_construction}, we are interested in showing that all  optimal invertible representations, no matter which, and no matter on which set $\mcZ'$, can be approximated using a representation with the following property: For every  $u\in S_0,v\in S_1$, there are at most $k$ points 
$z\in \mcZ$ that have $(u,v)$ as parents, see Figure \ref{fig:technical_details}(a). Here $k$ would depend only on the desired approximation degree, but not on $\mcZ'$, or on the exact representation we are approximating. We therefore refer to this result as the Uniform Approximation result. Its implications for practical use were discussed in Section \ref{sec:mifpo_construction}.

The notation used in this Section was introduced in the beginning of Section \ref{sec:proof_of_invertibility}.

To proceed with the analysis, in what follows we introduce the notion of two-point representation. The main result is given as Lemma \ref{lem:uniform_approx} below.

This Section uses the notation of Section \ref{sec:opt_reprs_properties}.
Let $T$ be an invertible representation,  let $u\in S_0, v\in S_1$ be some points, and denote by $\mcZ_{uv} = \Set{z^j}_{1}^k$ the set of all points $z\in \mcZ$ which have $u$ and $v$ as parents. Denote by 
\begin{equation}
    w_u = \sum_{j=1}^k \beta_0(u) T_0(z^j,u) \text{ and }
    w_v = \sum_{j=1}^k \beta_1(v) T_1(z^j,v)
\end{equation}
the total weights of $\beta_0$ and $\beta_1$ transferred by the representation from $u$ and $v$ respectively to $\mcZ_{uv}$. Recall that $\rho_u,\rho_v$ denote the distributions of $Y$ conditioned on $u,v$.
We call the situation above, i.e. the collection of numbers 
$\Brack{ \Set{\beta_0(u)T_0(z^j,u)}_{j\leq k}, \Set{\beta_1(v) T_1(Z^j,v)}_{j\leq k}}$,
a \emph{two point representation}, since it describes how the weight from the points $u,v$ is distributed in the representation,  independently of the rest of the representation. 
The contribution of $\mcZ_{uv}$ to the global performance cost is 
\begin{flalign}
    E_{uv,T} &:= \sum_{j\leq k}  \Prob{z^j} h(\Prob{Y\cond z^j}) \\ 
    &= 
    \sum_{j\leq k} \Brack{\alpha_0 \beta_0(u) T_0(z^j,u) + 
    \alpha_1 \beta_1(v) T_1(z^j,v)} h\Brack{
    \frac{\alpha_0 \beta_0(u) T_0(z^j,u) \rho_u + 
    \alpha_1 \beta_1(v) T_1(z^j,v)\rho_v}{
    \alpha_0 \beta_0(u) T_0(z^j,u) + 
    \alpha_1 \beta_1(v) T_1(z^j,v)    
    }
    },      
\end{flalign}
while its contribution to the fairness condition is 
\begin{equation}
\label{eq:two_point_fairness_condition}
    F_{uv,T} = \half \sum_{j\leq k} \Abs{\beta_0(u) T_1(z^j,u) - \beta_1(v) T_1(z^j,v)}.
\end{equation}
Let us now consider two extreme cases of two-point representations.
Assume that the total amounts of weight to be represented, $w_u,w_v$ are fixed. The first case is when $k=1$, and this is the maximum fairness case, since in this case the weights $w_u,w_v$ overlap as much as possible. Indeed, the contributions to the fairness penalty  and performance cost in this case are 
\begin{equation}
\label{eq:two_point_pareto_extr1}
    \Abs{w_u - w_v} \text{ and } \Brack{\alpha_0 w_u + \alpha_1 w_v} h\Brack{\frac{\alpha_0 w_u \rho_u + \alpha_1 w_v \rho_v}{
    \alpha_0 w_u + \alpha_1 w_v 
    }}
\end{equation}
respectively. The other extreme case is when $w_u$ and $w_v$ do not overlap at all. This case be realised with $k=2$, by sending all $w_u$ to $z^1$ and 
all $w_v$ to $z^2$. The fairness and performance contributions would be 
\begin{equation}
\label{eq:two_point_pareto_extr2}
    w_u + w_v \text{ and } \alpha_0 w_u \cdot  h\Brack{\rho_u}
    +\alpha_1 w_v \cdot  h\Brack{\rho_v},
\end{equation}
respectively. Note that the fairness penalty is the maximum possible, while the performance cost is the minimum possible (indeed, this is the cost before the representation, and any representation can only increase it, by Lemma \ref{lem:data_processing}). We thus observed that each two points 
$u,v$, with fixed total weight $w_u,w_v$, can have their own Pareto front of performance-fairness. One could, in principle, fix a threshold 
$\gamma_{uv}$, $\Abs{w_u-w_v}\leq \gamma_{uv} \leq w_u + w_v$ for the fairness penalty \eqref{eq:two_point_fairness_condition}, and obtain a performance cost between that in \eqref{eq:two_point_pareto_extr1} and \eqref{eq:two_point_pareto_extr2}. However, it is not clear how large the number of points $k$ should be in order to realise such intermediate representations. 
In the following Lemma we show that one can uniformly approximate all the points on the two-point Pareto front using a fixed number of points, that depends only on the function $h$. This means that in practice one can choose a certain number $n$ of $z$ points, and have guaranteed bounds on the possible amount of loss incurred with respect to all representations of all other sizes. 
\begin{lemma}
\label{lem:uniform_approx}
For every $\eps>0$, there a number $n=n_{\eps}$ depending only on the function $h$, with the following property:
For every two-point representation 
$\Brack{ \Set{\beta_0(u)T_0(z^j,u)}_{j\leq k}, \Set{\beta_1(v)T_1(Z^j,v)}_{j\leq k}}$,
with total weights $w_u,w_v$, there is a two point representation $T'$ on a set $\mcZ'_{u,v}$, with the same total weights, 
such that $\Abs{\mcZ'_{uv}} \leq n$, 
and such that 
\begin{equation}
\label{eq:lem_uniform_approx_eq1}
    F_{uv,T'} \leq F_{uv,T} \text{ and }\spaceo E_{uv,T'} \leq E_{uv,T} + 2(w_u+w_v)\eps.
\end{equation}
\end{lemma}


\begin{proof}
To aid with brevity of notation, define
for $j\leq k$
\begin{equation}
    c_0^j = \alpha_0 \beta_0(u) T_0(z^j,u), \spaceo
    c_1^j = \alpha_1 \beta_1(v) T_1(z^j,v).
\end{equation}
Then we can write 
\begin{equation}
    E_{uv,T} = \sum_{j\leq k} (c^j_0+c^j_1) \cdot h\Brack{\frac{c_0^j\rho_u + c_1^j \rho_v}{c_0^j + c_1^j} } , \spaceo 
    F_{uv,T} = \half \sum_{j\leq k} \Abs{\Lambda c^j},
\end{equation}
where $\Lambda$ is the vector $\Lambda = (\alpha_0^{-1},-\alpha_1^{-1})$, $c^j = (c_0^j,c_1^j)$, and $\Lambda c^j$ is the inner product of the two.  

Observe that the cost $E_{uv,T}$ depends on $c^j$ mainly through the fractions $\frac{c_0^j}{c_0^j + c_1^j}$. Our strategy thus would be to approximate all $k$ of such fractions by a $\delta$-net of a size independent of $k$. To this end, set 
\begin{equation}
    p^j = \frac{c_0^j}{c_0^j + c_1^j}
\end{equation}
and define $h_{uv}:[0,1] \rightarrow \RR$ by 
\begin{equation}
    h_{uv}(p) = h(p \rho_u + (1-p) \rho_v). 
\end{equation}
Since $h$ is continuous (by assumption), and defined on a compact set, it is \emph{uniformly} continuous, and so is $h_{uv}$. By definition, this means there is a $\delta>0$ such that for all $p,p'$ with $\Abs{p-p'}\leq \delta$, it holds that $\Abs{h_{uv}(p)-h_{uv}(p')}\leq \eps$. Let us now choose $\Set{x_i}_{i=1}^n$ to be a $\delta$ net on $[0,1]$. 
For every $i\leq n$ set 
\begin{equation}
    \Gamma_i = \Set{j \setsep \Abs{p^j-x_i} \leq \delta, \text{ and $i$ is minimal with this property}}.
\end{equation}
That is, $\Gamma_i$ is the set of indices $j$ such that $p^j$ is approximated by $x^i$. 
Using $x_i$ we construct the representation $T'$ as follows:
For $a\in \Set{0,1}$ set
\begin{equation}
    c'^i_a = \sum_{j\in \Gamma_i} c_a^j.
\end{equation}
For $n$ new points, $z^i \in \mcZ'_{uv}$, 
set $T'_0(z'^i,u) = c'^i_0/ \beta_0(u)$, 
$T'_1(z'^i,v) = c'^i_1/ \beta_1(v)$. Note that the total weights 
are preserved, $\sum_{i\leq n} c'^i_0 = w_u$ and $\sum_{i\leq n} c'^i_1 = w_v$. 

Next, for every $j\in \Gamma_i$ we have 
\begin{equation}
\Abs{\frac{c^j_0}{c^j_0 + c^j_1}  - x_i} \leq \delta. 
\end{equation}
Thus 
\begin{flalign}
\Abs{\frac{c'^i_0}{c'^i_0 + c'^i_1} - x_i} = 
\Abs{\frac{\sum_{j\in \Gamma_j} \SqBrack{c^j_0 -(c^j_0+c^j_1)x_i}}{c'^i_0 + c'^i_1}} \leq 
\frac{\sum_{j\in \Gamma_j} \delta (c^j_0 + c^j_1)}{c'^i_0 + c'^i_1} = \delta.
\label{eq:uniform_approx_lem_proof1}
\end{flalign}

Next, observe that by the construction of $x_i$,
\begin{flalign}
\Abs{E_{uv,T} - \sum_i (c'^i_0+c'^i_1) h_{uv}(x_i) } &= 
\Abs{\sum_i \sum_{j\in \Gamma_i} (c^j_0+c^j_1) h_{uv}(p^j) - \sum_i (c'^i_0+c'^i_1) h_{uv}(x_i)} \\
&\leq  
\sum_i \sum_{j\in \Gamma_i} (c^j_0+c^j_1) \eps \\
&= (w_u+w_v) \eps.
\end{flalign}
In addition, 
\begin{flalign}
\Abs{E_{uv,T'} - \sum_i (c'^i_0+c'^i_1) h_{uv}(x_i) } &= 
\Abs{\sum_i (c'^i_0+c'^i_1) h_{uv}(\frac{c'^i_0}{c'^i_0+c'^i_1}) - \sum_i (c'^i_0+c'^i_1) h_{uv}(x_i)} \\
&\leq   (w_u+w_v) \eps, 
\end{flalign}
where we have used \eqref{eq:uniform_approx_lem_proof1} in the last transition. 

Combining the two inequalities yields the second part of \eqref{eq:lem_uniform_approx_eq1},
\begin{equation}
    \Abs{E_{uv,T'} - E_{uv,T}} \leq  2(w_u+w_v)\eps.
\end{equation}

Finally, note that 
\begin{flalign}
\sum_{i}\Abs{\Lambda c'^i} &= 
\sum_{i}\Abs{ \sum_{j \in \Gamma_j} \Lambda c^j} \\ 
&\leq  
\sum_{i} \sum_{j \in \Gamma_j} \Abs{\Lambda c^j} \\ 
&=  
\sum_{j}  \Abs{\Lambda c^j} ,
\end{flalign}
yielding the first part of, and thus completing the proof of, statement \eqref{eq:lem_uniform_approx_eq1}.

It remains to observe that above we have used a $\delta$ net for $h_{uv}$, which depends on $\rho_u,\rho_v$. However, we can directly build an appropriate $\delta$-net in full range of $h$, the simplex $\Delta_{\mcY}$, 
which would produce bounds valid for all $u,v$. 
Indeed, let $\delta'$ be such that $\Abs{h(\nu)-h(\nu)}\leq \eps$ for all $\mu,\nu \in \Delta_{\mcY}$ with $\norm{u-v}_1 \leq \delta'$. Observe that the map $p \mapsto p \rho_v + (1-p) \rho_u$ is 2-Lipschitz from $\RR$ to $\Delta_{\mcY}$ equipped with the $\norm{\cdot}_1$ norm, for any $u,v\in \Delta_{\mcY}$. Thus, choosing $\delta = \half \delta'$, we have 
$\Abs{h_{uv}(p)-h_{uv}(p')}\leq \eps$ if $\Abs{p-p'}\leq \delta$.
This completes the proof of the Lemma. 
\end{proof}

\section{Concavity Of $E_r$}
\label{sec:concavity_of_Er}

Note that the variables $r$ appear in \eqref{eq:Er_perf_cost} both as coefficients multiplying $h$ and inside the arguments of $h$, in a fairly involved manner. Nevertheless, the cost turns out to still retain an interesting structure, as it is \emph{concave}, if $h$ is. We record this in the following Lemma. 

\begin{lemma}
\label{lem:concavity_of_mult_h}
If $h:\Delta_{\mcY} \rightarrow \RR$ is concave, then of every $\rho_1,\rho_2 \in \Delta_{\mcY}$ the function $g:\RR^2 \rightarrow \RR$, given by $g((c_1,c_2)) = (c_1+c_2) h(\frac{c_1 \rho_1 + c_2 \rho_2}{c_1+c_2})$ is concave.      
\end{lemma}

\begin{proof}
It is sufficient to show that for every $c,c'\in \RR^2$, we have $g((c+c')/2) \geq \half (g(c) + g(c')$. To this end, 
define the map $F:\RR^2 \rightarrow \Delta_{\mcY}$ by 
\begin{equation}
F(c) = \frac{c_1 \rho_1 + c_2 \rho_2}{c_1+c_2}
\end{equation}
and note that 
\begin{flalign}
F((c+c')/2) = F(c)\frac{c_1+c_2}{c_1+c_2+c'_1+c'_2} +  
F(c')\frac{c'_1+c'_2}{c_1+c_2+c'_1+c'_2}.
\end{flalign}
It then follows that 
\begin{flalign}
    g((c+c')/2) &= \half (c_1+c_2+c'_1+c'_2)h(F((c+c')/2)) \\
    &= \half (c_1+c_2+c'_1+c'_2) h\Brack{F(c)\frac{c_1+c_2}{c_1+c_2+c'_1+c'_2} +  
F(c')\frac{c'_1+c'_2}{c_1+c_2+c'_1+c'_2}} \\ 
    &\geq \half (c_1+c_2+c'_1+c'_2) \SqBrack{ \frac{c_1+c_2}{c_1+c_2+c'_1+c'_2} h(F(c)) + 
    \frac{c'_1+c'_2}{c_1+c_2+c'_1+c'_2}h(F(c'))  } \\ 
    &= \half \Brack{g(c) + g(c')}.
\end{flalign}
\end{proof}

\section{MIFPO Equality Constraint}
\label{sec:linearizing_fairness_constraint}

As noted in the main text, although the inequality constraint \eqref{eq:fairness_constraint_construction_notation} is convex in the variables $r$, and can be incorporated directly into most optimisation frameworks, it may be significantly more convenient to work with \emph{equality} constraints.  Using the following Lemma, we can find equivalent equality constraints in a particularly simple form. 
\begin{lemma}
\label{lem:tv_gamma_relation}
Let $\mu_0,\mu_1 \in \Delta_{\mcZ}$ be two probability distributions over $\mcZ$ and fix some $\gamma\geq 0$. If  $\norm{\mu_0 -\mu_1}_{TV} = \gamma$ then there exist $\phi_0,\phi_1\in \Delta_{\mcZ}$ such that $\mu_0+\gamma \phi_0 = \mu_1 + \gamma \phi_1$. In the other direction, if there exist 
$\phi_0,\phi_1\in \Delta_{\mcZ}$ such that $\mu_0+\gamma \phi_0 = \mu_1 + \gamma \phi_1$, then $\norm{\mu_0 -\mu_1}_{TV} \leq \gamma$.
\end{lemma}
The proof may be found in Section \ref{sec:proofs}.

As consequence of this result, if we find distributions 
$\phi_0,\phi_1\in \Delta_{\mcZ}$ such that $\mu_0+\gamma \phi_0 = \mu_1 + \gamma \phi_1$ holds, then we know that \eqref{eq:fairness_constraint_construction_notation} also holds, and conversely, if \eqref{eq:fairness_constraint_construction_notation} holds, then distributions as above exist. 

Using this observation, we introduce new variables, $\phi_{u,v,j}^0$  and $\phi_{u,v,j}^1$, for every $(u,v,j) \in \mcZ$, which correspond to  
$\phi_0((u,v,j))$ and $\phi_1((u,v,j))$ respectively. 
These variables will be required to satisfy the following constraints:
\begin{align}
\label{eq:phi_constr_1}
  &\phi_{u,v,j}^0 \geq 0, \spaceo \phi_{u,v,j}^1\geq 0  \spaceo \forall (u,v,j)\in \mcZ \\    
  & \sum_{u,v,j} \phi_{u,v,j}^0 = 1 \text{ and } \sum_{u,v,j} \phi_{u,v,j}^1 = 1 
\end{align}
\hspace{-2 pt}
\begin{equation}
\begin{aligned}
     \beta_0(u) r_{u,v,j}^0 &+ \gamma \phi_{u,v,j}^0 = 
      \beta_1(v) r_{v,u,j}^0 + \gamma \phi_{u,v,j}^1  \spaceo \forall (u,v,j) \in \mcZ.
    \label{eq:phi_constr_2}
\end{aligned}    
\end{equation}

Here the first two lines encode the fact that $\phi_0,\phi_1$ are probabilities, while the third line encodes the fairness constraint, as discussed above.

\section{Additional Proofs}
\label{sec:proofs}

Proof Of Lemma \ref{lem:tv_gamma_relation}
\begin{proof}
For this proof it is more convenient to work with the $\ell_1$ norm $\norm{\cdot}_1$ directly. Recall that $\norm{\mu_0 -\mu_1}_{TV} = \half \norm{\mu_0 -\mu_1}_{1}$ and that 
\begin{equation}
    \norm{\mu_0 -\mu_1}_{1} = \sum_{z} \Abs{\mu_0(z) -\mu_1(z)}.
\end{equation}

Assume that $\norm{\mu_0 -\mu_1}_{1} = 2\gamma$. Define the functions
$\bar{\phi}_0(z) = \Ind{\mu_1\geq \mu_0}(z) \cdot (\mu_1(z) - \mu_0(z))$ and 
$\bar{\phi_1}(z) = \Ind{\mu_0 \geq \mu_1}(z) \cdot (\mu_0(z) - \mu_1(z))$. Note that we then have
\begin{equation}
\label{eq:tv_gamma_rel1}
    \sum_z \bar{\phi}_0(z) = \sum_z \bar{\phi}_1(z) = \gamma.
\end{equation}
Indeed, define 
\begin{equation}
    \eta(z) = \Ind{\mu_0 \geq \mu_1}(z) \cdot \mu_1(z) + 
              \Ind{\mu_1 \geq \mu_0}(z) \cdot \mu_0(z). 
\end{equation}
Clearly, $\eta + \bar{\phi}_0 = \mu_1$ and thus 
$\sum_z \eta(z) + \bar{\phi}_0(z) = 1$.  
Similarly, 
$
\sum_z \eta(z) + \bar{\phi}_1(z) = 1
$.
Therefore we have
\begin{equation}
\label{eq:tv_gamma_rel2}
    \sum_z \bar{\phi}_0(z) = \sum_z \bar{\phi}_1(z).
\end{equation}
Note also that we can write 
\begin{equation}
2\gamma = \norm{\mu_0 -\mu_1}_{1} = \sum_z \Abs{\mu_0(z) - \mu_1(z)} = 
\sum_z \Brack{\bar{\phi}_0(z) +  \bar{\phi}_1(z)},     
\end{equation}
which combined with \eqref{eq:tv_gamma_rel2} yields \eqref{eq:tv_gamma_rel1}.

Next, we can also directly verify that 
\begin{equation}
    \mu_0 + \bar{\phi_0}  = \mu_1 + \bar{\phi_1},
\end{equation}
and thus setting $\phi_0 = \gamma^{-1} \bar{\phi_0},\phi_1 = \gamma^{-1} \bar{\phi_1}$ completes the proof of this direction.

In the other direction, given $\phi_0,\phi_1\in \Delta_{\mcZ}$ such that $\mu_0+\gamma \phi_0 = \mu_1 + \gamma \phi_1$, we have 
\begin{equation}
\sum_{z} \Abs{\mu_0(z) -\mu_1(z)} = \gamma \sum_{z} \Abs{\phi_1(z) -\phi_0(z)} \leq 2\gamma,
\end{equation}
thus completing the proof.
\end{proof}

\section{Experiments}
\label{sec:experiments_additional_detail}

This section describes additional evaluation details and experiments with fair classifiers. 
In Section\ref{sec:implementation_and_compute_details} we 
provide the main algorithm figure, and discuss technical implementation details. Section \ref{sec:supp_experiments_additional} contains the comparison to a number of fair classifiers, and Section \ref{sec:supp_concave_in_dccp} discusses implementation of the entropy cost $h$ within the DCCP framework. 

\subsection{Implementations and computational details}
\label{sec:implementation_and_compute_details}
    \begin{algorithm}[H]
      \caption{MIFPO Implementation}
      \label{alg:param_and_mifpo}
      \begin{algorithmic}
        \STATE {\bfseries Input:} Data $\Set{(x_i,a_i,y_i)}_{i\leq N}$, integers $L,k$. \\
        \spaceo For $a\in \Set{0,1}$ denote $X_a = \Set{x_i \setsep a_i=a}$.
        \STATE \textbf{1.} Learn \emph{calibrated} classifiers\\ \spaceo $c_0,c_1:\RR^d \rightarrow [0,1]$,  such that \\
        \spaceo $ c_a(x) \sim \Prob{Y=1\cond X=x, A=a} $   
        \STATE \textbf{2.} Construct the histograms
        $\Set{\beta_a(l)}_{l=1}^L$, \\ 
        \spaceo $a\in \Set{0,1}$, for the sets \\
           \spaceo $H_a = \Set{c_a(x) \setsep (x,a) \in D_a}\subset [0,1]$. \\
           \spaceo Choose bin representatives $\Set{\rho_l}_{l=1}^L$
        \STATE \textbf{3.} Solve MIFPO, given by Definition \ref{def:mifpo_def}, \\
        \spaceo with parameters $k$ and \\
        \spaceo $\Set{\beta_a(l)}_{l=1}^L$, $\Set{\rho_l}_{l=1}^L$, 
        $\alpha_a = \Abs{\Set{i\setsep a_i = a}}/N$.        
      \end{algorithmic}
   \end{algorithm}

The Pareto front evaluation requires two main parts - building a calibrated classifier required for evaluating $c_a = P(Y|X, A = a)$, and later solving the optimization problem MIFPO (see Algorithm \ref{alg:param_and_mifpo}).

For a calibrated classifier, we are using standard model calibration. Model calibration is a well-studied problem where we fit a monotonic function to the probabilities of some base model so that the probabilities will reflect real probabilities, that is, $P(Y|X)$. Here, we used Isotonic regression \citep{berta2024classifier} for model calibration with XGBoost \citep{chen2015xgboost} as the base model. For training the XGBoost model, a GridSearchCV approach is employed to find the best hyperparameters from a specified parameter grid, using 3-fold cross-validation.

The experiments were conducted on a system with an Intel Core i9-12900KS CPU (16 cores, 24 threads), 64 GB of RAM, and an NVIDIA GeForce RTX 3090 GPU.

\subsection{Additional Evaluations}
\label{sec:supp_experiments_additional}
\begin{figure*}
    \centering
    \includegraphics[width=\linewidth,height = 3.8cm]{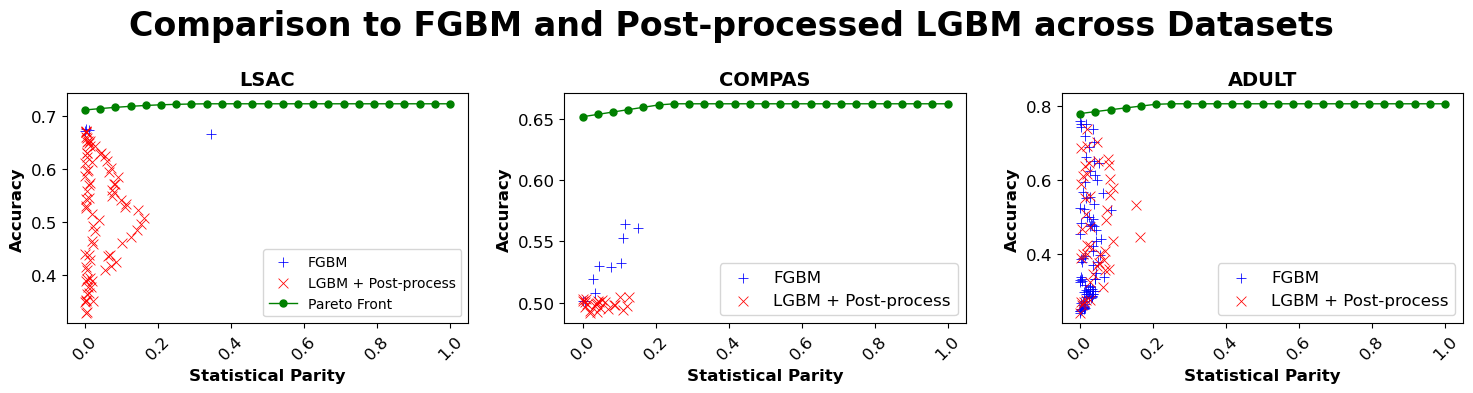}
    \caption{Comparing common fair classification pipelines to the MIFPO Pareto front, for LSAC, COMPAS, and ADULT datasets. 
    For FGBM and LGBM+Post Process methods, each point represents a trade-off obtained at a single hyper-parameter configuration.}
    \label{fig:pareto_classifiers}
\end{figure*}
Following the equivalence between the Pareto fronts of fair binary classification and representations for the accuracy cost (Section \ref{sec:problem_setting}), we evaluated the accuracy-fairness trade-off for some common fair classification algorithms. For our evaluation, we selected the most widely used algorithms based on GitHub repository stars and citation counts in the literature, demonstrating the importance of our proposed method in comparison to common approaches. We also evaluate FairFront, a more recent approach introduced in \cite{wang2023aleatoric}.

Fair classifiers generally fall into pre-processing, in-processing, and post-processing categories. Pre- and post-processing types often utilize standard classifiers as part of their fair classification pipeline. We specifically evaluated two widely adopted algorithms, representing different categories: FairGBM \cite{cruz2023fairgbmgradientboostingfairness}: An in-processing method where a boosting trees algorithm (LightGBM) is subjected to pre-defined fairness constraints. Balanced-Group-Threshold \cite{DBLP:journals/corr/abs-2111-04271}: A post-processing method which adjusts the threshold per group to obtain a fairness criterion.
For FairGBM, we used the original implementation provided by the authors. For Balanced-Group-Threshold post-processing, we utilized implementations available via Aequitas \cite{2018aequitas}, a popular bias and fairness audit toolkit. 

We conducted our evaluation using three of the most common datasets in this field, which are known to have relevance to real-world decision-making processes: the Adult dataset (income prediction), COMPAS (recidivism prediction), and LSAC (law school admission).

It is important to note that, as a rule, common fairness classification methods are not designed to control the fairness-accuracy trade-off explicitly. Instead, in most cases, these methods rely on rerunning the algorithm for a wide range of hyperparameter settings, in the hope that different hyperparameters would result in different fairness-accuracy trade-off points. However, there typically is no direct known and controlled relation between hyperpatameters and the obtained fairness-accuracy trade-off. 
For FairGBM, we utilized the hyperparameter ranges specified in the original paper,  \cite{cruz2023fairgbmgradientboostingfairness}. In the case of the balancing post-processing method, we conducted a grid search over the full range of all possible hyperparameters to ensure a comprehensive analysis.

Figure \ref{fig:pareto_classifiers} shows the MIFPO computed Pareto front, and  all hyperparameter runs of the two algorithms above, with accuracy evaluated on the test set.  These experiments demonstrate the following two points: \textbf{(a)}
The standard classifiers achieve a considerably lower accuracy than what is theoretically possible at a given fairness level. \textbf{(b)} the existing methods are also unable to present solutions for the full range of the statistical parity values. The values from the FGBM and the post-processing algorithms all have statistical parity $\leq 0.2$. 
Similarly to to the case of fair representations, these results emphasize the limitations of current fair classifiers in achieving optimal trade-offs between accuracy and fairness across the full range of fairness values.

\begin{figure}
    \centering
    \includegraphics[width=1\linewidth]{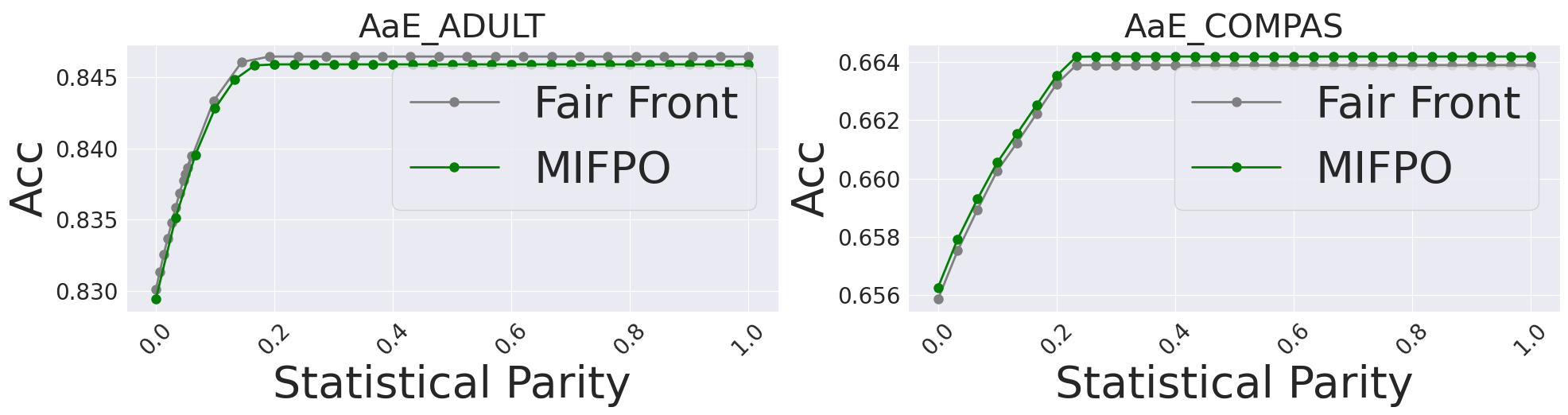}
    \caption{Comparing FairFront to MIFPO accuracy-fairness tradeoff on two curated datasets.}
    \label{fig:fairfront}
\end{figure}

Additionally, we add a comparison to FairFront \citep{rebut1}, which, similarly to our work,  depends on estimates of $P(Y|X)$. See the discussion in Section \ref{sec:literature}. We note, however, that the utility of this comparison is limited due to the very strict setup defined in that paper, which cannot be applied to natural datasets. Specifically: 1) The implementation in \cite{wang2023aleatoric} requires creating a finite discrete space \emph{by binning over the full $\RR^d$ feature space}\footnote{This is completely different from the discretizaton of $\Delta_{\mcY}$ used in MIFPO.   The space $\RR^d$ is the feature space, and is much larger than $\Delta_{\mcY}$.}, which allows for perfect modeling of the distribution $P(Y|X)$ (by simple counting). This is not normally possible on real datasets in practice. 2) The binning itself was performed manually for each dataset separately, and we were unable to discern the logic behind the selected parameters. 3) To allow reasonable binnig, the true dimensions of the data are manually reduced, again by picking features manually for each dataset.  Due to these issues, it was practically impossible to apply the method in \cite{wang2023aleatoric} to  other standard datasets, or even to the datasets used in \cite{wang2023aleatoric} but with standard features. 

Nonetheless, for the sake of comparison, we
compared MIFPO and FairFront on the same restricted, preprocessed version of the datasets that FairFront used, with the same bin-based probability estimator. The results are shown in Figure \ref{fig:fairfront}.


\subsection{Minimization of Concave Functions under Convex Constraints and the Entropy Loss}
\label{sec:supp_concave_in_dccp}

As described in figure \ref{sec:experiments}, we used the disciplined convex concave programming (DCCP) framework and the associated solver, \citep{shen2016disciplined, disciplined_github} for solving the concave minimization with convex constraints problem. 

Minimizing concave functions under convex constraints is a common problem in optimization theory. Unlike convex optimization where global minima can be readily found, in concave minimization problems we only know that the local minimas lie on the boundaries of the feasible region defined by the convex constraints. While techniques such as branch-and-bound algorithms, cutting plane methods, and heuristic approaches are often employed, here we used the framework of DCCP which gain a lot of popularity in recent years.

The DCCP framework extends disciplined convex programming (DCP) to handle nonconvex problems with objective and constraint functions composed of convex and concave terms. The idea behind a "disciplined" methodology for convex optimization is to impose a set of conventions inspired by basic principles of convex analysis and the practices of convex optimization experts. These "disciplined" conventions, while simple and teachable, allow much of the manipulation and transformation required for analyzing and solving convex programs to be automated. DCCP builds upon this idea, providing an organized heuristic approach for solving a broader class of nonconvex problems by combining DCP principles with convex-concave programming (CCP) methods, and is implemented as an extension to the CVXPY package in Python.

While convenient, the use of the disciplined framework bears some limitations. Mainly, generic operations like element-wise multiplication are not under the allowed set of operations (and for obvious reasons), which limits the usability. Notice, that for the prediction accuracy measure $h(p) = \min(p, 1-p)$ this is not a problem, but for the entropy classification error $h((1-p, p)) = -p \log p - (1-p) \log(1-p)$, this is more challenging. Nevertheless, here we show that the standard DCCP framework allows for entropy classification error.

\begin{lemma}
Let $a = (1 - \alpha) \mathcal{V}(v) r_{v,z}$ and $b = \alpha \cdot \mathcal{U}(u) r_{u,z}$, with $p_v = p_a$ and $p_u = p_b$.

We can write the cost function under the entropy accuracy error as:
\begin{align*}
(a + b) \cdot \mathrm{entropy}&\left( \frac{a \cdot p_a + b \cdot p_b}{a + b} \right) \\ &=  \mathrm{entropy}(a \cdot (1-p_a) + b \cdot (1 - p_b)) + \mathrm{entropy}(a \cdot p_a + b \cdot p_b) - \mathrm{entropy}(a + b)
\end{align*}
\end{lemma}
\begin{proof}
\begin{align*}
    \mathrm{entropy}(x) &= \mathrm{entropy}(x, 1 - x) = x \left[ \log(x) - \log(1 - x) \right] + \log(1 - x)
\end{align*}

Thus, 
\begin{align*}
    \mathrm{entropy}\left( \frac{a \cdot p_a + b \cdot p_b}{a + b} \right) &= \frac{a \cdot p_a + b \cdot p_b}{a + b} \left[ \log\left( \frac{a \cdot p_a + b \cdot p_b}{a + b} \right) - \log\left( 1 - \frac{a \cdot p_a + b \cdot p_b}{a + b} \right) \right] \\
    &\quad + \log\left( 1 - \frac{a \cdot p_a + b \cdot p_b}{a + b} \right) \\ 
    &= \frac{a \cdot p_a + b \cdot p_b}{a + b} \left[ \log(a \cdot p_a + b \cdot p_b) - \log(a + b - a \cdot p_a - b \cdot p_b) \right] \\
    &\quad + \log(a + b - a \cdot p_a - b \cdot p_b) - \log(a + b)
\end{align*}

\noindent Hence :
\begin{align*}
    (a + b) \cdot \mathrm{entropy}\left( \frac{a \cdot p_a + b \cdot p_b}{a + b} \right) &= (a + b) \cdot \left[ \frac{a \cdot p_a + b \cdot p_b}{a + b} \left[ \log(a \cdot p_a + b \cdot p_b) - \log(a + b - a \cdot p_a - b \cdot p_b) \right] \right] \\
    &\quad + (a + b) \cdot \log(a + b - a \cdot p_a - b \cdot p_b) \\
    &\quad - (a + b) \cdot \log(a + b)
\end{align*}

\begin{align*}
    &= \mathrm{entropy}(a \cdot p_a + b \cdot p_b) - (a \cdot p_a + b \cdot p_b) \cdot \log(a + b - a \cdot p_a - b \cdot p_b) \\
    &\quad + (a + b) \cdot \log(a + b - a \cdot p_a - b \cdot p_b) - \mathrm{entropy}(a + b)
\end{align*}

Finally, the expression can be written as:
$$ = \mathrm{entropy}(a + b - a \cdot p_a - b \cdot p_b) + \mathrm{entropy}(a \cdot p_a + b \cdot p_b) - \mathrm{entropy}(a + b)$$

$$ = \mathrm{entropy}(a \cdot (1-p_a) + b \cdot (1 - p_b)) + \mathrm{entropy}(a \cdot p_a + b \cdot p_b) - \mathrm{entropy}(a + b)$$
\end{proof}

Given the element-wise entropy function $x \cdot (1-x)$ is with known characteristics and under the dccp framework, we can use the entropy error for our cost using : 
\begin{align*}
  =& \mathrm{entropy}((1 - \alpha) \mathcal{V}(v) r_{v,z} \cdot (1-p_v) + \alpha \cdot \mathcal{U}(u) r_{u,z} \cdot (1 - p_u))   \\
  & -\mathrm{entropy}((1 - \alpha) \mathcal{V}(v) r_{v,z} \cdot p_v + \alpha \cdot \mathcal{U}(u) r_{u,z} \cdot p_u) \\
  &  - \mathrm{entropy}((1 - \alpha) \mathcal{V}(v) r_{v,z} + \alpha \cdot \mathcal{U}(u) r_{u,z})
\end{align*}

\section{Fair Classifiers As Fair Representations}
\label{supp:fair_classifiers}

As discussed in Section \ref{sec:problem_setting}, 
Pareto front of binary classifiers with statistical parity can be computed from the Pareto front of representations with total variation fairness distance. In this Section we provide the proof of this result, Lemma \ref{lemma:cls}. The Lemma and the related notation are restated here for convenience. 

For a binary classifier $\hat{Y}$ of $Y$, its prediction error is defined as usual by $\epsilon(\hat{Y}) := \Prob{\hat{Y} \neq Y}$. 
The statistical parity \emph{distance} of $\hat{Y}$ is defined as 
\begin{equation}
    \Delta_{SP}(\hat{Y}) := \Abs{\Prob{\hat{Y} = 1\cond A = 1} - \Prob{\hat{Y} = 1 \cond A = 0} }.
\end{equation}
Let the uncertainy measure $h$ be defined by \eqref{eq:accuraccy_h_definition}. Note that the first part of the Lemma does use the special properties of this $h$ and does not necessarily hold for other costs $h$. 
\begin{lemma}
    Let $\hat{Y}$ be a classifier of $Y$. Then there is a representation 
    given by a random variable $Z$ on a set $\mcZ$ with $\Abs{\mcZ} = 2$, such that 
\begin{equation}
\label{eq:class_repr_lem_eq_1}
    \Expsubidx{z \sim Z}{h(\Prob{Y \cond Z = z})} \leq \epsilon(\hat{Y}) 
    \text{ and } \norm{\mu_0 - \mu_1}_{TV} \leq \Delta_{SP}(\hat{Y}). 
\end{equation}    
Conversely, for any given representation $Z$, there is a classifier 
$\hat{Y}$ of $Y$ as a function of $Z$ (and thus of $(X,A)$), such that 
\begin{equation}
\label{eq:class_repr_lem_eq_2}
    \epsilon(\hat{Y}) \leq \Expsubidx{z \sim Z}{h(\Prob{Y \cond Z = z})}   
    \text{ and } \Delta_{SP}(\hat{Y}) \leq \norm{\mu_0 - \mu_1}_{TV}.     
\end{equation}
\end{lemma}
\begin{proof}
Let us begin with the second part of the Lemma, inequalities \eqref{eq:class_repr_lem_eq_2}.
Given a representation $Z$, $\epsilon(\hat{Y}) \leq \Expsubidx{z \sim Z}{h(\Prob{Y \cond Z = z})}$ follows since $\Expsubidx{z \sim Z}{h(\Prob{Y \cond Z = z})}$ is the error of the optimal classifier of $Y$ as a function of $Z$. We choose $\hat{Y}$ to be such an optimal classifier and thus satisfy the above inequality, with equality. 
Next, the second inequality in \eqref{eq:class_repr_lem_eq_2} holds for 
 for \emph{any} classifier $\hat{Y}$ derived from $Z$. 
The argument below is a slight generalisation of the argument in \cite{madras2018learning}. 
Define $f(z) = \Prob{\hat{Y} =1 \cond Z=z}$. 
Note that for $a \in \Set{0,1}$, $\Prob{\hat{Y} =1 \cond A=a} = \int f(z) d\mu_a(z)$. Thus 
\begin{flalign}
    \Prob{\hat{Y} =1 \cond A=0} - \Prob{\hat{Y} =1 \cond A=1} &= 
    \int f(z) d\mu_0(z) - \int f(z) d\mu_1(z) \\ 
    &\leq \sup_{g \cond \Abs{g}\leq 1}
     \Abs {\int g(z) d\mu_0(z) - \int g(z) d\mu_1(z) } \\ 
     &= \norm{\mu_0 - \mu_1}_{TV},     
\end{flalign}
where we have used $\Abs{f} \leq 1$ in the second line.
Repeating the argument also for $\Prob{\hat{Y} =1 \cond A=1} - \Prob{\hat{Y} =1 \cond A=0}$, we obtain the second inequality in \eqref{eq:class_repr_lem_eq_2}.

We now turn to the first statement, \eqref{eq:class_repr_lem_eq_1}. 
Let $\hat{Y}$ be a classifier of $Y$ as a function of $(X,A)$. 
Observe that thus by definition  $\Prob{\hat{Y} \cond X,A}$  induces a distribution on the set 
$\Set{0,1}$, and thus may be considered as a representation $Z := \hat{Y}$ of $(X,A)$ on that set. We now relate the properties of this $Z$ as a representation to the quantities $\epsilon(\hat{Y})$ and $\Delta_{SP}(\hat{Y})$. Similarly to the argument above, the first part of \eqref{eq:class_repr_lem_eq_1} follows since  
$\Expsubidx{z \sim Z}{h(\Prob{Y \cond Z = z})}$ is the best possible error over all classifiers.  
For the second part, note that since $\hat{Y}$ is binary, we have 
\begin{equation}
\label{eq:lem_cls_repr_binary_yhat}
    \Prob{\hat{Y} = 1\cond A=0} - \Prob{\hat{Y} = 1\cond A=1} = 
    -\Prob{\hat{Y} = 0\cond A=0} + \Prob{\hat{Y} = 0\cond A=1}. 
\end{equation}
It follows that 
\begin{flalign}
    \norm{\mu_0 - \mu_1}_{TV} &= \half 
    \sum_{v \in \Set{0,1}} 
    \Abs{\Prob{\hat{Y} = v\cond A=0} - \Prob{\hat{Y} = v\cond A=1}} \\ 
    &= \Abs{\Prob{\hat{Y} = 1\cond A=0} - \Prob{\hat{Y} = 1\cond A=1}}  \\ 
    &= \Delta_{SP}(\hat{Y}),
\end{flalign}
where we have used \eqref{eq:lem_cls_repr_binary_yhat} for the second to third line transition. This completes the proof of the second part of \eqref{eq:class_repr_lem_eq_1}.
\end{proof}

\section{Computational Complexities}
\label{sec:computational_complexities}

In this Section we discuss alternative discretization schemes for the MIFPO algorithm.  We also discuss various complexity  aspects of the classification algorithms \cite{xian2023fair} and \cite{wang2023aleatoric} and relate them to the complexity of MIFPO.

In Section \ref{sec:practical_construction_details} we described a data independent discretization of $\Delta_\mcY$ by binning. While effective for small label sets $\mcY$, larger sets would require a different approach. 
One alternative is to cluster the data instead of binning $\Delta_{\mcY}$ itself. Indeed, by choosing cluster centers 
$\Set{\eta_i}_{1}^M \subset \Delta_{\mcY}$, such that each data point $\Set{f^*(x,a)}_{(x,a)\in D}$ is well approximated by one the centers (or just most points are approximated), we can guarantee arbitrarily good approximation of the true Pareto front. 
The cardinality $M$ of such clustering would depend on the intrinsic dimension of the data, which we typically expect to be lower than the full dimension of $\Delta_{\mcY}$, due to the Manifold Hypothesis, \cite{fefferman2016testing}.

Another possibility is to not use any explicit discretization, and instead to use each data point as a separate bin (equivalently, we use the points themselves as the cluster centers $\Set{\eta_i}$). In this case, the complexity scales with the size of the dataset, but not with the dimensions of $\Delta_{\mcY}$. The MIFPO construction in Section \ref{sec:mifpo_construction} implies that MIFPO in this case would have $O(N^2)$ variables for a dataset of size $N$. While not applicable for large $N$, this is similar to the complexity of a variety of often used algorithms. 
Classical example of such complexity is the Spectral Clustering. We also observe that \cite{xian2023fair}, a recent state of the art fairness classification algorithm mentioned above, also has such complexity.

Indeed, the approach in \cite{xian2023fair} involves the computation of a certain transportation plan between data points, and the encoding of such plans also requires $O(N^2)$ variables. 
Thus problem sizes occuring in MIFPO would be smaller or equal to those in \cite{xian2023fair}, despite the fact that MIFPO is solving a considerably more general problem (see Section \ref{sec:literature}).

Finally, the algorithm in \cite{wang2023aleatoric} involves 
optimization in the space of confusion matrices, with 
dimensions of size $\Abs{\mcA}\cdot \Abs{\mcY}\times \Abs{\mcY}$. As discussed in Section \ref{sec:literature}, the reduction of the problem to confusion matrices of possible due to special properties of the classification problem and the accuracy loss. 

The algorithm  in \cite{wang2023aleatoric}, FairFront, is an iterative algorithm,  where each iteration involves solving a certain \emph{difference-of-convex} (DC) program which is constructed from a full dataset. The class of DC programs is equivalent to that convex-concave programs considered in this paper (see \cite{shen2016disciplined}). In fact, similarly to MIFPO, the algorithm in \cite{wang2023aleatoric} also uses DCCP, \cite{shen2016disciplined}, although applied to a different problem. 

In each iteration, the solution of DC program is then used to add new constraints to a certain main convex program. While it is proved that asymptotically this process converges to the optimal front, there are no bounds on the number of iterations. This may lead to the convex solver crashing due to too many constraints, and in fact we have observed such crashes in our evaluation.

To summarize \footnote{In this Section we have discussed the theoretical complexity aspects of FairFront. Additional details pertaining to the official implementation of \cite{wang2023aleatoric} may be found in Section \ref{sec:supp_experiments_additional}.}, MIFPO  involves solving \emph{one} convex-concave problem, with size which may be independent of the data size. 
In contrast, FairFront involves iteratively solving convex-concave problems and a main convex program, where the number of terms in the objective of each convex-concave problem scales with the size of the data, and the number of constraints grows 
in the convex problem grows with iterations, thus making the iterations progressively harder.

\section{Factorization}
\label{sec:factorization_proofs}

In this Section we the factorization result, Lemma \ref{lem:factorization}. We restate the result for convenience. 

Recall that $f^*(x,a)$ denotes the Bayes optimal classifier of $Y$, i.e. $f^* : \mcX \times \mcA \rightarrow \Delta_{\mcY}$ is given by  
\begin{equation}
    f^*(x,a) := \Prob{Y = \cdot \cond X=x,A=a}.
\end{equation}

We define a new variable $X'$, with values in $\Delta_{\mcY}$, 
by $X' := f^*(X,A)$. 

\begin{lemma}[Factorization]
\label{lem:factorization_ver_proof}
    For any representation $Z$ of $(X,A)$, there is a representation $Z'$ of $(X',A)$, such that 
    \begin{equation}
        \Expsubidx{z' \sim Z'}{h(\Prob{Y \cond Z' = z'})} = 
        \Expsubidx{z \sim Z}{h(\Prob{Y \cond Z = z})} 
        \text{ and }
        D_{TV}(Z') \leq D_{TV}(Z).
    \end{equation}
\end{lemma}
\begin{proof}
The representation $Z'$ of $X',A$  will be defined as follows:
For $\sigma \in \Delta_{\mcY}, a\in \mcA, z\in \mcZ$ set:
\begin{flalign}
    \Prob{Z' = z \cond X' = \sigma, A=a} &:= 
    \Prob{Z = z \cond X' = \sigma, A=a}  \\ 
    &= \Prob{Z = z \cond f^*(X,a) = \sigma, A=a}
\end{flalign}
where the second line is the definition and is added for clarity. 
To see intuition behind this definition note that we have
\begin{flalign}
  \Prob{Z' = z \cond X' = \sigma, A=a} &=  
  \Prob{Z = z \cond f^*(X,a) = \sigma, A=a}  \label{eq:eq:factr_proof_minus2} \\ 
  &=\sum_{x \in \mcX} \Prob{Z = z \cond X=x, f^*(x,a) = \sigma, A=a} 
  \Prob{X = x \cond f^*(x,a) = \sigma, A=a} \notag \\ 
  &=\sum_{x\in \mcX} \Prob{Z = z \cond X=x, A=a} 
  \Prob{X = x \cond f^*(x,a) = \sigma, A=a} \label{eq:eq:factr_proof_minus1}
\end{flalign}
In words, for a fixed $a$, to compute $\Prob{Z' = z \cond X' = \sigma, A=a}$ we effectively collect all $x$ such that $f^*(x,a) = \sigma$ and average all of their representations.  Equivalently, all points $x$ with the same $\sigma$ are merged into one point, and their representations summed up according to their relative weight. 

We will now show that neither the performance  nor the fairness condition change under this operation. 

Since $D_{TV}(Z)$ is defined solely in terms of the distributions $\Prob{Z = \cdot \cond A = a}$, 
to show that $D_{TV}(Z) = D_{TV}(Z')$ it is enough to show that  
\begin{equation}
\Prob{Z = \cdot \cond A = a} = \Prob{Z' = \cdot \cond A = a} 
\spaceo \forall a \in \mcA.
\end{equation}
To this end, we have 
\begin{flalign}
 \Prob{Z' = z \cond A = a} &= 
 \sum_{\sigma} 
 \Prob{Z' = z \cond X' = \sigma, A = a} \Prob{X' = \sigma \cond A = a} \label{eq:factr_proof_1} \\
 &= \sum_{\sigma} \Prob{Z = z \cond X' = \sigma, A = a} \Prob{X' = \sigma \cond A = a}  \label{eq:factr_proof_2}  \\
 &=  \Prob{Z = z \cond A = a}. 
\end{flalign}
where the transition \eqref{eq:factr_proof_1} to \eqref{eq:factr_proof_2} is by the definition of $Z'$. 
Thus we have shown that $D_{TV}(Z) = D_{TV}(Z')$.

Next, note that the above  argument implies also that  $\Prob{Z=z} = \Prob{Z'=z}$. Thus, in order to show  the performance equality, 
\begin{equation}
\Expsubidx{z' \sim Z'}{h(\Prob{Y \cond Z' = z'})} = 
        \Expsubidx{z \sim Z}{h(\Prob{Y \cond Z = z})}, 
\end{equation}
it is enough to show that $\Prob{Y \cond Z' = z} = \Prob{Y \cond Z = z}$ for every $z\in \mcZ$. Further, again since 
$\Prob{Z=\cdot} = \Prob{Z'=\cdot }$, we can show that 
\begin{equation}
    \Prob{Y=y , Z' = z} = \Prob{Y=y , Z = z}
\end{equation}
for all $z\in \mcZ,y\in \mcY$.
Write 
\begin{flalign}
    &\Prob{Y=y, Z = z} \label{eq:factr_proofz_1}\\ 
    &= \sum_{a} \sum_{x} \Prob{Y=y, X=x, A=a, Z = z} \label{eq:factr_proofz_2}\\ 
    &= 
    \sum_{a} \sum_{x} \Prob{Y=y \cond X=x, A=a, Z = z}  
    \Prob{X=x, A=a, Z = z} \label{eq:factr_proofz_3}\\
    &= \sum_{a} \sum_{x} \Prob{Y=y \cond X=x, A=a, Z = z}  
    \Prob{Z=z \cond X=x, A=a} \Prob{X=x, A=a} \label{eq:factr_proofz_4}\\     
    &= \sum_{a} \sum_{x} \Prob{Y=y \cond X=x, A=a}  
    \Prob{Z=z \cond X=x, A=a} \Prob{X=x, A=a} \label{eq:factr_proofz_5}\\ 
    &= \sum_{a} \sum_{\sigma}
    \sum_{x | f^*(x,a) = \sigma} \Prob{Y=y \cond X=x, A=a}  
    \Prob{Z=z \cond X=x, A=a} \Prob{X=x, A=a} \label{eq:factr_proofz_6}\\ 
    &= \sum_{a} \sum_{\sigma}
    \sigma(y) \Prob{X' = \sigma, A= a}
    \sum_{x | f^*(x,a) = \sigma}   
    \Prob{Z=z \cond X=x, A=a} 
    \frac{\Prob{X=x, A=a}}{\Prob{X' = \sigma, A= a}} \label{eq:factr_proofz_7}\\
    &= \sum_{a} \sum_{\sigma}
    \sigma(y) \Prob{X' = \sigma, A= a}
    \sum_{x | f^*(x,a) = \sigma}   
    \Prob{Z=z \cond X=x, A=a} 
    \Prob{X=x \cond X' = \sigma, A= a)} \label{eq:factr_proofz_8}\\ 
    &= \sum_{a} \sum_{\sigma}
    \sigma(y) \Prob{X' = \sigma, A= a}
    \Prob{Z' = z \cond  X' = \sigma, A= a} \label{eq:factr_proofz_9}\\
    &= \sum_{a} \sum_{\sigma}
    \sigma(y) \Prob{Z' = z,  X' = \sigma, A= a} \label{eq:factr_proofz_10} \\
    &= \sum_{a} \sum_{\sigma}
    \Prob{Y=y \cond X' = \sigma, A = a} \Prob{Z' = z,  X' = \sigma, A= a} \label{eq:factr_proofz_11} \\
    &= \Prob{Y=y , Z'=z}. \label{eq:factr_proofz_12}
\end{flalign}
Here, the transition 
\eqref{eq:factr_proofz_4}-\eqref{eq:factr_proofz_5} is due to the independence condition \eqref{eq:representation_y_indep_condition}. 
On line \eqref{eq:factr_proofz_6} we split the sum over $x$ into sum over subsets $\Set{x \setsep f^*(x,a) = \sigma}$ and 
an outer some over all $\sigma$. The transition 
\eqref{eq:factr_proofz_8}-\eqref{eq:factr_proofz_9} is due to
the equality \eqref{eq:eq:factr_proof_minus1}-\eqref{eq:eq:factr_proof_minus2}. 
Finally, for the transition 
\eqref{eq:factr_proofz_9}-\eqref{eq:factr_proofz_10}, we have used the fact that 
\begin{equation}
\label{eq:factr_proof_final}
    \sigma(y) = \Prob{Y=y \cond X' = \sigma, A= a},
\end{equation}
which holds by definition of $X'$. 
Similarly to the earlier discussion on merging of $x$ with similar value of $f^*$, the above argument proceeded by 
regrouping the summation over $x$ by the value of $f^*(x,a)$.
The computation thus showed that this process yields the definition of $Z'$. In particular, this regrouping process 
and equation \eqref{eq:factr_proof_final} explain why the space 
$\Delta_{\mcY}$ is special and all representations may be factored through it. 
This completes the proof of the Lemma. 
\end{proof}

In the above argument we have used the summation over $\sigma$, 
i.e. $\sum_{\sigma} \ldots$. This is formally possible when $(X,A)$ has a discreet distribution. The full general case may be obtained simply by replacing the summation by integration and conditioning on $\sigma$ by the general conditional expectation operator. 

\end{document}